\documentclass[10pt,twocolumn,letterpaper]{article}

\usepackage{cvpr}
\usepackage{times}
\usepackage{epsfig}
\usepackage{graphicx}
\usepackage{amsmath}
\usepackage{amssymb}
\usepackage{comment}

\usepackage[pagebackref=true,breaklinks=true,letterpaper=true,colorlinks,bookmarks=false]{hyperref}

\cvprfinalcopy 

\ifcvprfinal\fi

\DeclareMathOperator*{\plim}{plim}

\begin{document}

\title{Learning Weighted Submanifolds with Variational Autoencoders and Riemannian Variational Autoencoders}

\author{Nina Miolane\\
Stanford University\\
{\tt\small nmiolane@stanford.edu}
\and
Susan Holmes\\
Stanford University\\
{\tt\small susan@stat.stanford.edu}
}

\maketitle



\begin{abstract}

Manifold-valued data naturally arises in medical imaging. In cognitive neuroscience, for instance, brain connectomes base the analysis of coactivation patterns between different brain regions on the analysis of the correlations of their functional Magnetic Resonance Imaging (fMRI) time series – an object thus constrained by construction to belong to the manifold of symmetric positive definite matrices. One of the challenges that naturally arises in these studies consists of finding a lower-dimensional subspace for representing such manifold-valued and typically high-dimensional data. Traditional techniques, like principal component analysis, are ill-adapted to tackle non-Euclidean spaces and may fail to achieve a lower-dimensional representation of the data – thus potentially pointing to the absence of lower-dimensional representation of the data. However, these techniques are restricted in that: (i) they do not leverage the assumption that the connectomes belong on a pre-specified manifold, therefore discarding information; (ii) they can only fit a linear subspace to the data. In this paper, we are interested in variants to learn potentially highly curved submanifolds of manifold-valued data. Motivated by the brain connectomes example, we investigate a latent variable generative model, which has the added benefit of providing us with uncertainty estimates – a crucial quantity in the medical applications we are considering. While latent variable models have been proposed to learn linear and nonlinear spaces for Euclidean data, or geodesic subspaces for manifold data, no intrinsic latent variable model exists to learn nongeodesic subspaces for manifold data. This paper fills this gap and formulates a Riemannian variational autoencoder with an intrinsic generative model of manifold-valued data. We evaluate its performances on synthetic and real datasets by introducing the formalism of weighted Riemannian submanifolds.
\end{abstract}

\section{Introduction}

Representation learning aims to transform data $x$ into a lower-dimensional variable $z$ designed to be more efficient for any downstream machine learning task, such as exploratory analysis of clustering, among others. In this paper, we focus on representation learning for manifold-valued data that naturally arise in medical imaging. Functional Magnetic Resonance Imaging (fMRI) data are often summarized into ``brain connectomes", that capture the coactivation of brain regions of subjects performing a given task (memorization, image recognition, or mixed gamble task, for example). As correlation matrices, connectomes belong to the cone of symmetric positive definite (SPD) matrices. This cone can naturally be equipped with a Riemannian manifold structure, which has shown to improve performances on classification tasks \cite{Barachant2013ClassificationApplications}. Being able to learn low-dimensional representations of connectomes within the pre-specified SPD manifold is key to model the intrinsic variability across subjects, and tackle the question: do brain connectomes from different subjects form a lower-dimensional subspace within the manifold of correlation matrices? If so, each subject's connectome $x$ can be represented by a latent variable $z$ of lower dimension. Anticipating potential downstream medical tasks that predict behavioral variables (such as measures of cognitive, emotional, or sensory processes) from $z$, we seek a measure of uncertainty associated with $z$. In other words, we are interested in a posterior in $z$ given $x$.

While the literature for generative models capturing lower-dimensional representations of Euclidean data is rich, such methods are typically ill-suited to the analysis of manifold-valued data. Can we yet conclude that lower-dimensional representations within these manifolds are not achievable? The aforementioned techniques are indeed restricted in that: either (i) they do not leverage any geometric knowledge as to the known manifold to which the data, such as the connectomes, belong; or (ii) they can only fit a linear (or geodesic, \textit{i.e.} the manifold equivalent of linear) subspace to the data. In this paper, we focus on alternatives with a latent variable generative model that address (i) and (ii).

\subsection{Related Work}

There is a rich body of literature on manifold learning methods. We review here a few of them, which we evaluate based on the following desiderata:
\begin{itemize}
    \item Is the method applicable to manifold-valued data?
    \item For methods on Euclidean data: does the method learn a linear or a nonlinear manifold, see Figure~\ref{fig:submanifolds} (a, b)? 
    \item For methods geared towards Riemannian manifolds: does the method learn a geodesic (\textit{i.e.} the manifold equivalent of a linear subspace) - or a nongeodesic subspace, see Figure~\ref{fig:submanifolds} (c, d)?
    \item Does the method come with a latent variable generative model?
\end{itemize}

\begin{figure}[h!]
\centering
\def\svgwidth{1.05\columnwidth}
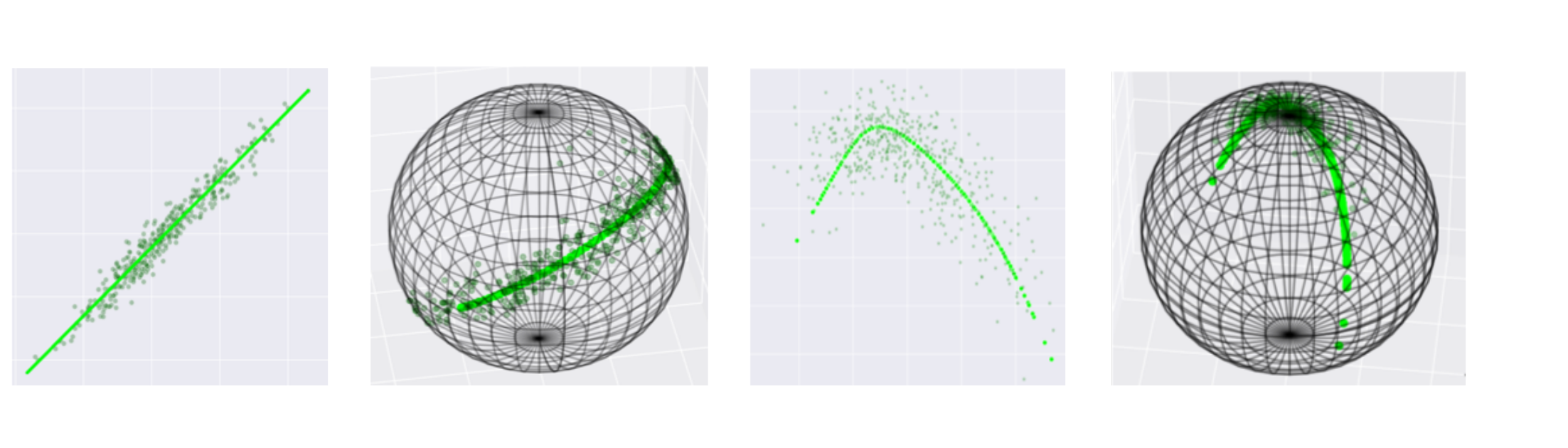
\caption{(a) Learning a 1D linear subspace in a 2D Euclidean space; (b) Learning a geodesic in a 2D manifold (sphere); (c) Learning a 1D nonlinear subspace in a 2D Euclidean space; (d) Learning a nongeodesic 1D subspace in a 2D manifold (sphere).\label{fig:submanifolds}}
\end{figure}


\subsubsection{Learning Linear and Geodesic Subspaces} 

Principal Component Analysis (PCA) \cite{Pearson1901Space} learns a linear subspace,  while Probabilistic PCA (PPCA) and Factor Analysis (FA) \cite{Tipping1999ProbabilisticAnalysis} achieve the same goal within a probabilistic framework relying on a latent variable generative mode; see Figure~\ref{fig:submanifolds} (a). These techniques are based on vector space's operations that make them unsuitable for data on manifolds. As a consequence, researchers have developed methods for manifold-valued data, which take into account the geometric structure; see Figure~\ref{fig:submanifolds} (b). 

Principal Geodesic Analysis (PGA) \cite{Fletcher2004, Sommer2014OptimizationAnalysis}, tangent PGA (tPGA) \cite{Fletcher2004}, Geodesic Principal Component Analysis (gPCA) \cite{Huckemann2010IntrinsicActions}, principal flows \cite{Panaretos2014PrincipalFlows}, barycentric subspaces (BS) \cite{Pennec2018BarycentricManifolds} learn variants of ``geodesic" subspaces, \textit{i.e.} generalizations in manifolds of linear spaces in Euclidean spaces. Probabilistic PGA \cite{Zhang2013ProbabilisticAnalysis} achieves the same goal, while adding a latent variable model generating data on a manifold. 

However, these methods are restricted in the type of submanifold that can be fitted to the data, either linear or geodesic - a generalization of linear subspaces to manifolds. This restriction can be considered both a strength and a weakness. While it protects from overfitting with a submanifold that is too flexible, it also prevents the method from capturing possibly nonlinear effects. With current dataset sizes exploding (even within biomedical imaging datasets which have been historically much smaller), it seems that the investigation of flexible submanifold learning techniques takes on crucial importance.

\subsubsection{Learning Non-Linear and Nongeodesic Subspaces} 

While methods for learning nonlinear manifolds from Euclidean data are numerous (see Figure~\ref{fig:submanifolds} (c)), those providing a latent variable generative models are scarce. Kernel PCA \cite{Schoelkopf1998NonlinearProblem}, multi-dimensional scaling and its variants \cite{Cox2000MultidimensionalScaling, Bronstein2006GeneralizedMatching}, Isomap \cite{Tenenbaum2000}, Local Linear Embedding (LLE) \cite{Roweis2000NonlinearEmbedding}, Laplacian eigenmaps \cite{Belkin2003LaplacianRepresentation}, Hessian LLE \cite{Donoho2003HessianData}, Maximum variance unfolding \cite{Weinberger2006AnUnfolding}, and others, learn lower-dimensional representations of data but do not provide a latent variable generative model, nor a parameterization of the recovered subspace.
 
In contrast, principal curves and surfaces (PS) \cite{Hastie1989PrincipalCurves} and autoencoders fit a nonlinear manifold to the data, with an explicit parameterization of this manifold. However, this framework is not directly transferable to non-Euclidean data and has been more recently generalized to principal curves on Riemannian manifolds \cite{Hauberg2016PrincipalManifolds}. To our knowledge, this is the only method for nongeodesic submanifold learning on Riemannian manifolds (see Figure~\ref{fig:submanifolds} (d)). A probabilistic approach to principal curves was developed in \cite{Chang2001ASurfaces} for the Euclidean case, but not the manifold case. Similarly, variational autoencoders (VAEs) \cite{Kingma2014Auto-EncodingBayes} were developed to provide a latent variable generative model for autoencoders. However, they do not apply to manifold-valued data.

In order to create a latent variable generative model for manifold-valued data, we can either generalize principal curves on manifolds by adding a generative model or generalize VAEs for manifold-valued data. Principal curves require a parameterization of the curve that involves a discrete set of points. As the number of points needed grows exponentially with the dimension of the estimated surface, scaling this method to high dimensional principal surfaces becomes more difficult. As a consequence, we choose to generalize VAEs to manifold-valued data. This paper introduces Riemannian VAE, an intrinsic method that provides a flexible generative model of the data on a pre-specified manifold. We emphasize that our method does not amount to embedding the manifold in a larger Euclidean space, training the VAE, and projecting back onto the original manifold - a strategy that does not come with an intrinsic generative model of the data. We implement and compare both methods in Section~\ref{sec:comp}.

\subsection{Contribution and Outline}

This paper introduces the intrinsic Riemannian VAE, a submanifold learning technique for manifold-valued data. After briefly reviewing the (Euclidean) VAE, we present our Riemannian generalization. We show how Riemannian VAEs generalize both VAE and Probabilistic Principal Geodesic Analysis. We provide theoretical results describing the family of submanifolds that can be learned by the Riemannian method. To do so, we introduce the formalism of weighted Riemannian submanifolds and associated Wasserstein distances. This formalism also allows giving a sense to the definition of consistency in the context of submanifold learning. We use this to study the properties of VAE and Riemannian VAE learning techniques, on theoretical examples and synthetic datasets. Lastly, we deploy our method on real data by applying it to the analysis of connectome data.


\section{Riemannian Variational Autoencoders (rVAE)}

\subsection{Review of (Euclidean) VAE}


We begin by setting the basis for variational autoencoders (VAEs) \cite{Kingma2014Auto-EncodingBayes, Rezende2014StochasticModels}. Consider a dataset $x_1, ..., x_n \in \mathbb{R}^D$. A VAE models each data point $x_i$ as the realization of a random variable $X_i$ generated from a nonlinear probabilistic model with lower-dimensional latent variable $Z_i$ taking value in $\mathbb{R}^L$, where $L < D$, such as:
\begin{equation}
    X_i = f_\theta(Z_i) + \epsilon_i,
\end{equation}
where $Z_i  \sim N(0, \mathbb{I}_L)$ \textit{i.i.d.} and $\epsilon_i$ represents \textit{i.i.d.} measurement noise distributed as $\epsilon_i \sim N(0, \sigma^2 \mathbb{I}_D)$. The function $f_\theta$ belongs to a family $\mathcal{F}$ of nonlinear generative models parameterized by $\theta$, and is typically represented by a neural network, called the decoder, such that: $f_\theta(\bullet) = \Pi_{k=1}^K g(w_k \bullet + b_k)$ where $\Pi$ represents the composition of functions, $K$ the number of layers, $g$ an activation function, and the $w_k, b_k$ are the weights and biases of the layers. We write: $\theta = \{w_k, b_k\}_{k=1}^K$. This model is illustrated on Figure~\ref{fig:vaemodel}.


The VAE pursues a double objective: (i) it learns the parameters $\theta$ of the generative model of the data; and (ii) it learns an approximation $q_\phi(z|x)$, within a variational family $\qdists$ parameterized by $\phi$, of the posterior distribution of the latent variables. The class of the generative model $\mathcal{F}$ and the variational family $\qdists$ are typically fixed, as part of the design of the VAE architecture. The VAE achieves its objective by maximizing the evidence lower bound (ELBO) defined as:
\begin{align}
     \elbo 
     & = \mathbb{E}_{q_\phi(z|x)} \left[ \log \frac{\joint}{q_\phi(z|x)} \right]
\end{align}
which can conveniently be rewritten as:
\begin{align*}
    \elbo 
    &= \logl - \text{KL} \left( q_\phi(z|x) \parallel \post \right)\\
    &= \mathbb{E}_{q_\phi(z)} \left[ \log p_\theta(x|z) \right] - \text{KL} \left( q_\phi(z|x) \parallel \prior \right)\\
    &=\mathcal{L}_{\text{rec}}(x, \theta, \phi) + \mathcal{L}_{\text{reg}}(x, \phi),
\end{align*}
where the terms $\mathcal{L}_{\text{rec}}(x, \theta, \phi)$ and $\mathcal{L}_{\text{reg}}(x, \phi)$ are respectively interpreted as a reconstruction objective and as a regularizer to the prior on the latent variables.

From a geometric perspective, the VAE learns a manifold $\hat N = N_{\hat{\theta}} = f_{\hat{\theta}} (\mathbb{R}^L)$ designed to estimate the true submanifold of the data $N_\theta = f_\theta (\mathbb{R}^L)$. The approximate distribution $q_\phi(z|x)$ can be seen as a (non-orthogonal) projection of $x$ on the subspace $N_{\hat \theta}$ with associated uncertainty.



\begin{figure}[h]
\centering
\begin{center}
\includegraphics[page=2,width=16cm]{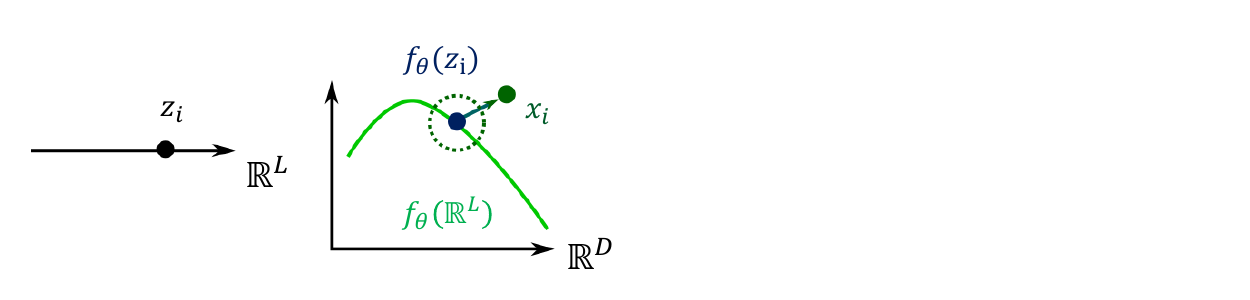}
\end{center}
\caption{Generative model for the variational autoencoder with latent space $\mathbb{R}^L$ and data space $\mathbb{R}^D$. The latent variable $z_i$ is sampled from a standard multivariate normal distribution on $\mathbb{R}^L$ and embedded into $\mathbb{R}^D$ through the embedding $f_\theta$. The data $x_i$ is generated by addition of a multivariate isotropic Gaussian noise in $\mathbb{R}^D$.\label{fig:vaemodel}}
\end{figure}

\subsection{Riemannian VAE (rVAE)}

We generalize the generative model of VAE for a dataset $x_1, . . . , x_n$ on a Riemannian manifold $M$. We need to adapt two aspects of the (Euclidean) VAE: the embedding function $f_\theta$ parameterizing the submanifold, and the noise model on the manifold $M$. We refer to supplementary materials for details on Riemannian geometry, specifically the notions of Exponential map, Riemannian distance and Fr\'echet mean.

\subsubsection{Embedding}

Let $\mu \in M$ be a base point on the manifold. We consider the family of functions $f_\theta : \mathbb{R}^L \mapsto \mathbb{R}^D \simeq T_\mu M$ that are parameterized by a fully connected neural network of parameter $\theta$, as in the VAE model. We define a new family of functions with values on $M$, by considering: $f^M_{\mu, \theta}(\bullet) = \text{Exp}^M(\mu, f_\theta(\bullet))$ as a embedding from $\mathbb{R}^L$ to $\mathbb{R}^D$, where $\text{Exp}^M(\mu, \bullet)$ is the Riemannian exponential map of $M$ at $\mu$.

\subsubsection{Noise model}

We generalize the Gaussian distribution from the VAE generative model, as we require a notion of distribution on manifolds. There exist several generalizations of the Gaussian distribution on Riemannian manifolds \cite{Pennec2006}. To have a tractable expression to incorporate into our loss functions, we consider the minimization of entropy characterization of \cite{Pennec2006}:
\begin{equation}
    p(x|\mu, \sigma) = \frac{1}{C(\mu, \sigma)}\exp \left( - \frac{d(\mu, x)^2}{2\sigma^2}\right),
\end{equation}
where $C(\mu, \sigma)$ is a normalization constant:
\begin{equation}
    C(\mu, \sigma) = \int_M  \exp \left( - \frac{d(\mu, x)^2}{2\sigma^2}\right)dM(x),
\end{equation}
and $dM(x)$ refers to the volume element of the manifold $M$ at $x$. We call this distribution an (isotropic) Riemannian Gaussian distribution, and use the notation $x \sim N^M(\mu, \sigma^2\mathbb{I}_D)$. We note that this noise model could be replaced with a different distribution on the manifold $M$, for example a generalization of a non-isotropic Gaussian noise on $M$.

\subsubsection{Generative model}

We introduce the generative model of Riemannian VAE (rVAE) for a dataset $x_1, ..., x_n$ on a Riemannian manifold $M$:
\begin{equation}\label{model:gvae}
    X_i | Z_i = N^M \left(\text{Exp}^M(\mu, f_\theta(Z_i)), \sigma^2 \right) \text{ and }  Z_i \sim N(0, \mathbb{I}_L),
\end{equation}
where $f_\theta$ is represented by a neural network and allows to represent possibly highly ``non-geodesic" submanifolds. This model is illustrated on Figure~\ref{fig:gvaemodel}. 

From a geometric perspective, fitting this model learns a submanifold $N_{\hat{\theta}} = \text{Exp}^M(\mu, f_{\hat{\theta}} (\mathbb{R}^L))$ designed to estimate the true $N_{\theta} = \text{Exp}^M(\mu, f_{\theta} (\mathbb{R}^L))$ in the manifold $M$. The approximate distribution $q_\phi(z|x)$ can be seen as a (non-orthogonal) projection of $x$ on the submanifold $N_{\hat \theta}$ with associated uncertainty.

\begin{figure}[h]
\centering
\begin{center}
\includegraphics[page=3,width=16cm]{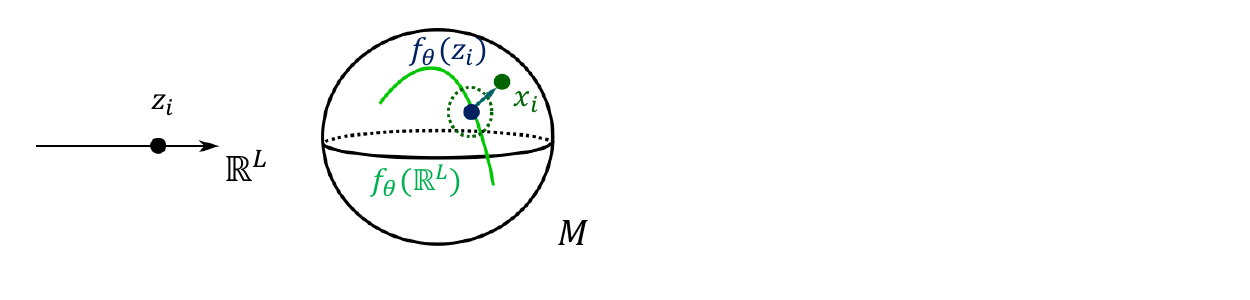}
\end{center}
\caption{Generative model for the Riemannian variational autoencoder with latent space $\mathbb{R}^L$ and data space $M$. The latent variable $z_i$ is sampled from a standard multivariate normal distribution on $\mathbb{R}^L$ and embedded into $M$ through the embedding $f_{\mu, \theta}$. The data $x_i$ is generated by addition of a Riemannian multivariate isotropic Gaussian noise in $M$.\label{fig:gvaemodel}}
\end{figure}

\subsubsection{Link to VAE and PPGA}

The rVAE model is a natural extension of both the VAE and the Probabilistic PGA (PPGA) models. We recall that, for $M = \mathbb{R}^D$, the Exponential map is an addition operation, $\text{Exp}^{\mathbb{R}^D}(\mu, y) = \mu + y$. Furthermore, the Riemannian Gaussian distribution reduces to a multivariate Gaussian $N^{\mathbb{R}^D}(\mu, \sigma^2\mathbb{I}_D)=N(\mu, \sigma^2\mathbb{I}_D)$. Thus, the Riemannian VAE model coincides with the VAE model when $M = \mathbb{R}^D$. Furthermore, the Riemannian VAE model coincides with the model of PPGA: 
\begin{equation}\label{eq:ppga}
    X_i|Z_i \sim N^M \left(\text{Exp}^M(\mu, WZ_i), \sigma^2\right) \text{ and } Z_i \sim N(0, \mathbb{I}_L),
\end{equation}
when the decoder is a linear neural network: $f_\theta(z) = Wz$ for $z \in \mathbb{R}^L$. 

Inference in PPGA was originally introduced with a Monte Carlo Expectation Maximization (MCEM) scheme in \cite{Zhang2013ProbabilisticAnalysis}. In contrast, our approach fits the PPGA model with variational inference, as we will see in Section~\ref{sec:fit}. Variational inference methods being less accurate but faster than Monte-Carlo approaches, our training procedure represents an improvement in speed to the PPGA original inference method, at the cost of some accuracy.

\section{Expressiveness of rVAE}\label{sec:approx}

The Riemannian VAE model parameterizes an embedded submanifold $N$ defined by a smooth embedding $f_\theta^M$ as:
\begin{equation}
N = f^M_\theta(\mathbb{R}^L) =\text{Exp}^M(\mu, f_\theta(\mathbb{R}^L)),
\end{equation}
where $f_\theta$ is the function represented by the neural net, with a smooth activation function, and the parameter $\mu$ is absorbed in the notation $\theta$ in $f^M_\theta$. The flexibility in the nonlinear function $f_\theta$ allows rVAE to parameterize embedded manifolds that are not necessarily geodesic at a point. A question that naturally arises is the following: can rVAE represent any smooth embedded submanifold $N$ of $M$? We give results, relying on the universality approximation theorems of neural networks, that describe the embedded submanifolds that can be represented with rVAE.

\subsection{Weighted Riemannian submanifolds}

We introduce the notion of weighted submanifolds and suggest the associated formalism of Wasserstein distances to analyze dissimilarities between general submanifolds of $M$ and submanifolds of $M$ parameterized by rVAE.

\begin{definition}[Weighted (sub)manifold]
Given a complete $N$-dimensional Riemannian manifold $(N, g^N)$ and a smooth probability distribution $\omega : N \rightarrow \mathbb{R}$, the weighted manifold $(N, \omega)$ associated to $N$ and $\omega$ is defined as the triplet:
\begin{equation}
    (M, g^N, d\nu = \omega. dN),
\end{equation}
where $dN$ denotes the Riemannian volume element of $N$.
\end{definition}

The Riemannian VAE framework parameterizes weighted submanifold defined by:
\begin{equation}
    N_\theta: (f^M_\theta(\mathbb{R}^L), g_M, f^M_\theta \ast N(0, \mathbb{I}_L)),
\end{equation}
so that the submanifold $N_\theta$ is modeled as a singular (in the sense of the Riemannian measure of $M$) probability density distribution with itself as support. The distribution is associated with the embedding of the standard multivariate Gaussian random variable $Z \sim N(0, \mathbb{I}_L)$ in $M$ through $f^M_\theta$. 

\subsection{Wasserstein distance on weighted submanifolds}

We can measure distances between weighted submanifolds through the Wasserstein distances associated with their distributions.

\begin{definition}[Wasserstein distance]
The 2-Wasserstein distance between probability measures $\nu_1$ and $\nu_2$ defined on $M$, is defined as:
\begin{equation}
    d_2(\nu_1, \nu_2) = 
    \left( 
    \inf_{\gamma \in \Gamma(\mu, \nu)} 
    \int_{M\times M} d_M(x_1, x_2)^2d\gamma(z_1, z_2)
    \right)^{1/2}
\end{equation}
where $\Gamma (\nu_1 , \nu_2)$ denotes the collection of all measures on $M \times M$ with marginals $\nu_1$ and $\nu_2$ on the first and second factors respectively.
\end{definition}

Wasserstein distances have been introduced previously in the context of variational autoencoders with a different purpose: \cite{Tolstikhin2018WassersteinAuto-Encoders} use the Wasserstein distance with any cost function between the observed data distribution and the learned distribution, penalized with a regularization term, to train the neural network. In contrast, we use the Wasserstein distance with the square of the Riemannian distance as the cost function to evaluate distances between submanifolds. Therefore, we evaluate a distance between the data distribution and the learned distribution before the addition of the Gaussian noise. We do not use this distance to train any model; we only use it as a performance measure.

\subsection{Weighted submanifold approximation result}

The following result describes the expressiveness of rVAEs.

\begin{proposition}
Let $(N, \nu)$ be a weighted Riemannian submanifold of $M$, embedded in a submanifold of $M$ homeomorphic to $\mathbb{R}^L$ for which there exists an embedding $f$ that verifies: $\nu = f \ast \mu_T$ where $\mu_T$ is a truncated standard multivariate normal on $\mathbb{R}^L$. Let assume the existence of $\mu \in M$ such that $N \subset V(\mu)$, where $V(\mu)$ is the domain of bijection of the Riemannian exponential of $M$ at $\mu$. Then, for any $0 < \epsilon < 1$, there exists a Riemannian VAE with decoder $f_\theta$ parameterized by $\theta$ such that:
\begin{equation}
    d_{2}(N, N_\theta) < \epsilon
\end{equation}
where $d_2$ is the 2-Wasserstein distance for the weighted submanifolds.
\end{proposition}


\begin{proof}
The proof is provided in the supplementary materials.
\end{proof}



As Hadamard manifolds are homeomorphic to $\mathbb{R}^L$ through their Riemannian Exponential map, the assumption $N \subset V(\mu)$ is always verified in their case. This suggests that it can be better to equip a given manifold with a Riemannian metric with negative curvature. In the case of the space of SPD matrices in Section~\ref{sec:spd}, we choose a metric with negative curvature.




\section{Learning and inference for rVAEs}\label{sec:fit}

We show how to train rVAE by performing learning and inference in model~(\ref{model:gvae}). 

\subsection{Riemannian ELBO}

As with VAE, we use stochastic gradient descent to maximize the ELBO:
\begin{align*}
    \elbo 
    & = \mathcal{L}_{\text{rec}}(x, \theta, \phi) + \mathcal{L}_{\text{reg}}(x, \phi) \\ 
    &= \mathbb{E}_{q_\phi(z)} \left[ \log p_\theta(x|z) \right] - \text{KL} \left( q_\phi(z|x) \parallel \prior \right)
\end{align*}
where the reconstruction objective $\mathcal{L}_{\text{rec}}(x, \theta, \phi)$ and the regularizer $\mathcal{L}_{\text{reg}}(x, \phi)$ are expressed using probability densities from model~(\ref{model:gvae}), and a variational family chosen to be the multivariate Gaussian:
\begin{align*}
q_\phi(z|x) &= N(h_\phi(x), \sigma_\phi^2(x)),\\
\quad p(z) &= N(0, \mathbb{I}_L),\\
\quad p(x|z) &= N^M(\text{Exp}^M(\mu, f_\theta(Z_i)), \sigma^2\mathbb{I}_D ).
\end{align*}

The reconstruction term writes:
\begin{align*}
 \mathcal{L}_{\text{rec}}(x, \theta, \phi)
     &= \int_z \Bigl( - \log C(\sigma^2, r(\mu, z, \theta))\\
    &\qquad\qquad-  \frac{d_M(x, \text{Exp}(\mu, f_\theta(z)))^2}{2\sigma^2} \Bigr) q_\phi(z|x) dz,
\end{align*}
while the regularizer is:
\begin{align*}
 \mathcal{L}_{\text{reg}}(x, \phi)
    &= \int_z \log \frac{q_\phi(z|x)}{p(z)} q_\phi(z|x)dz\\
    &= \frac{1}{2} \sum_{l=1}^L \left( 
        1 + \log(\sigma_l^{(i)})^2 - (\mu_l^{(i)})^2 - (\sigma_l^{(i)})^2
        \right),
\end{align*}
where $C$ is the normalization constant, that depends on $r(\mu, z, \theta)$, to the injectivity radius of the Exponential map at the point $\text{Exp}(\mu, f_\theta(z))$ \cite{Postnikov2001}. We note that, although in the initial formulation of the VAE, the $\sigma$ depends on $z$ and $\theta$ and should be estimated during training, the implementations usually fix it and estimate it separately. We perform the same strategy here. In practice, we use the package \texttt{geomstats} \cite{GEOMSTATS} to plug-in the manifold of our choice within the rVAE algorithm.

\subsection{Approximation}

To compute the ELBO, we need to perform an approximation as providing the exact value of the normalizing constant $C(\sigma^2, r(\mu, z, \theta))$ is not trivial. The constant $C$ depends on the $\sigma^2$ and the geometric properties of the manifold $M$, specifically the injectivity radius $r$ at $\mu$. 

For Hadamard manifolds, the injectivity radius is constant and equal to $\infty$, thus $C = C(\sigma)$ depends only on $\sigma$. As we do not train on $\sigma$, we can discard the constant $C$ in the loss function. For non-Hadamard manifolds, we consider the following approximation of the $C$, that is independent of the injectivity radius:
\begin{equation}
    C = \frac{1 + O(\sigma^3) + O(\sigma / r)}{\sqrt{(2\pi)^D\sigma^{2D}}}.
\end{equation}
This approximation is valid in regimes with $\sigma^2$ low in comparison to the injectivity radius, in other words, when the noise's standard deviation is small in comparison to the distance to the cut locus from each of the points on the submanifold. After this approximation, we can discard the constant $C$ from the ELBO as before.




\subsection{An important remark}

We highlight that our learning procedure does not boil down to projecting the manifold-valued data onto some tangent space of $M$ and subsequently applying a Euclidean VAE. Doing so would implicitly model the noise on the tangent space as a Euclidean Gaussian, as shown in the supplementary materials. Therefore, the noise would be modulated by the curvature of the manifold. We believe that this is an undesirable property, because it entangles the probability framework with the geometric prior, \textit{i.e.} the random effects with the underlying mathematical model. 

\section{Goodness of fit for submanifold learning}

We consider the goodness of fit of rVAEs (and VAEs) using the formalism of weighted submanifolds that we introduced in Section~\ref{sec:approx}. In other words, assuming that data truly belong to a submanifold $N_\theta = f_{\mu, \theta}(\mathbb{R}^L)$ and are generated with the rVAE model, we ask the question: how well does rVAE estimate the true submanifold, in the sense of the 2-Wasserstein distance? For simplicity, we consider that rVAE is trained with a latent space $\mathbb{R}^L$ of the true latent dimension $L$. Inspired by the literature of curve fitting \cite{Chernov2011}, we define the following notion of consistency for weighted submanifolds.

\begin{definition}[Statistical consistency]
We call the estimator $N_{\hat \theta}$ of $N_{\theta}$ statistically consistent if:
\begin{equation}
    \plim_{n \rightarrow +\infty} d_{W_2}(N_{\hat \theta} , N_{\theta} ) = 0.
\end{equation}
\end{definition}

Denoting $N_{\hat \theta}$ the submanifold learned by rVAE, we want to evaluate the function: $d(n, \sigma) = d_{W_2}(N_{\hat \theta} , N_{\theta})$, for different values of $n$ and $\sigma$, where $\hat \theta$ depends on $n$ and $\sigma$.

\subsection{Statistical inconsistency on an example}

We consider data generated with the model of probabilistic PCA (PPCA) with $\mu=0$ \cite{Tipping1999ProbabilisticAnalysis}, \textit{i.e.} a special case of a rVAE model:
\begin{equation}\label{eq:model1d}
    X_i = w Z_i + \epsilon_i 
\end{equation}
where: $w \in \mathbb{R}^{D\times L}$, $Z \sim N(0, \mathbb{I}_L)$ \textit{i.i.d.} and $\epsilon \sim N(0, \mathbb{I}_D)$ \textit{i.i.d.}. We train a rVAE, which is a VAE in this case, on data generated by this model. We chose a variational family of Gaussian distributions with variance equal to $1$. Obviously, this is not the learning procedure of choice in this situation. We use it to illustrate the behavior of rVAEs and VAEs.

The case $D=1$ and $L=1$ allows to perform all computations in closed forms (see supplementary materials). We compute the distance between the true and learned submanifold in terms of the 2-Wasserstein distance:
\begin{align*}
    d_2(\nu_{\theta}, \nu_{\hat \theta}) = w - \sqrt{\frac{\hat\sigma^2}{2} - 1} \rightarrow w - \sqrt{\frac{w^2 -1 }{2}} \neq 0 
\end{align*}
where $\hat \sigma^2$ is the sample variance of the $x_i$'s. We observe that the 2-Wasserstein distance does not converge to $0$ as $n \rightarrow + \infty$ if $w \neq 1$ or $-1$. This is an example of statistical inconsistency, in the sense that we defined in this section.

\subsection{Experimental study of inconsistency}

We further investigate the inconsistency with synthetic experiments and consider the following three Riemannian manifolds: the Euclidean space $\mathbb{R}^2$, the sphere $S^2$ and the hyperbolic plane $H_2$. The definitions of these manifolds are recalled in the supplementary materials. We consider three Riemmanian VAE generative models respectively on $\mathbb{R}^2$, $S^2$ and $H_2$, with functions $f_\theta$ that are implemented by a three layers fully connected neural network with softplus activation. Figure~\ref{fig:synth} shows synthetic samples of size $n=100$ generated from each of these models. The true weighted 1-dimensional submanifold corresponding to each model is shown in light green.

\begin{figure}[h!]
    \centering
    \includegraphics[height=2.6cm]{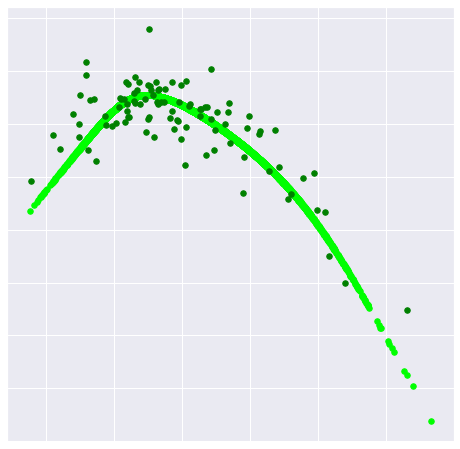}
    \includegraphics[height=2.6cm]{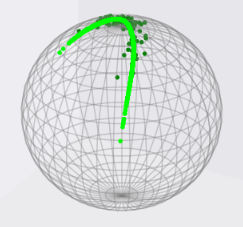}
    \includegraphics[height=2.6cm]{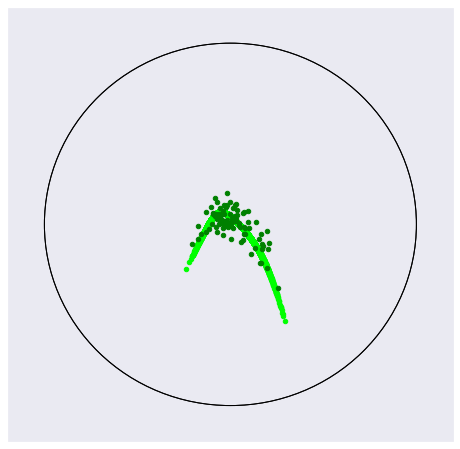}
    \caption{Synthetic data on the manifolds $\mathbb{R}^2$ (left), $S^2$ (center) and $H_2$ in its Poincar\'e disk representation (right). The light green represents the true weighted submanifold, the dark green points represents data points generated with rVAE.\label{fig:synth}}
\end{figure}

For each manifold, we generate a series of datasets with sample sizes $n \in \left\{10, 100\right\}$ and noise standard deviation such that $\log\sigma^2 \in \left\{-6, -5, -4, -3, -2 \right\}$. For each manifold and each dataset, we train a rVAE with the same architecture than the decoder that has generated the data, and standard deviation fixed to a constant value.

Figure~\ref{fig:goodness_gvae} shows the 2-Wasserstein distance between the true and the learned weighted submanifold in each case, as a function of $\sigma$, where different curves represent the two different values of $n$. These plots confirm the statistical inconsistency observed in the theoretical example. For $\sigma \neq 0$,the VAE and the rVAE do not converge to the submanifold that has generated the data as the sample size increases. This observation should be taken into consideration when these methods are used for manifold learning, \textit{i.e.} in a situation where the manifold itself is essential. Other situations that use these methods only as a way to obtain lower-dimensional codes may or may not be affected by this observation.

\begin{figure}[h!]
    \centering
    \includegraphics[width=8.2cm]{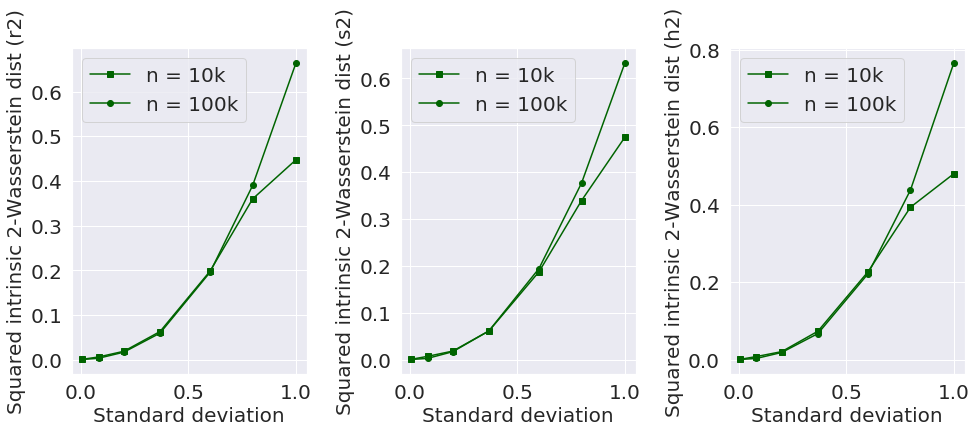}
    \caption{Goodness of fit for submanifold learning using the 2-Wasserstein distance.  First column: $\mathbb{R}^2$; Second column: $S^2$; Third column: $H_2$.}
    \label{fig:goodness_gvae}
\end{figure}

Additionally, we observe that this statistical inconsistency translates into an asymptotic bias that leads rVAEs and VAEs to estimate flat submanifolds, see Figure~\ref{fig:flat}. We provide an interpretation to a statement in \cite{Shao2018TheModels}, where the authors compute the curvature of the submanifold learned with a VAE on MNIST data and observe a ``surprinsingly little" curvature. Our experiments indicate that the true submanifold possibly has some curvature, but that its estimation does not because of noise regime around the submanifold is ``too high". Interesting, this remark challenges the very assumption of the existence of a submanifold: if the noise around the manifold is large, does the manifold assumption still hold?

\begin{figure}[h!]
    \centering
    \includegraphics[height=2.6cm]{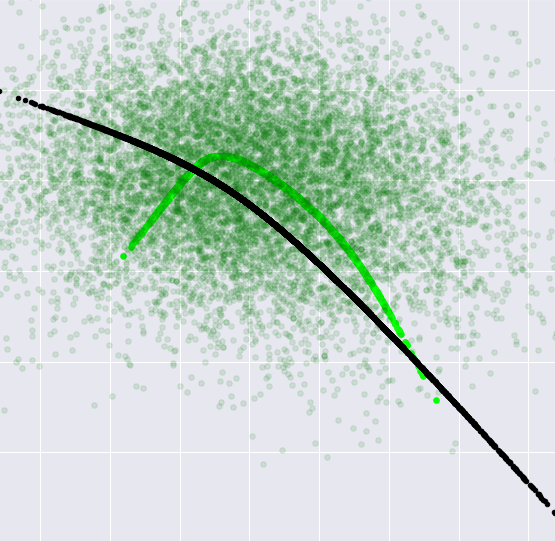}
    \includegraphics[height=2.6cm]{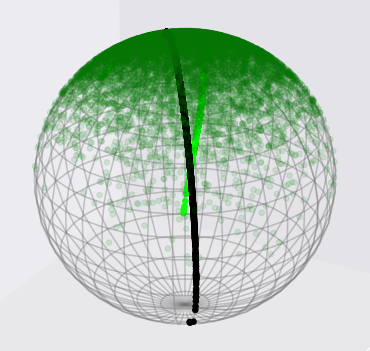}
    \includegraphics[height=2.6cm]{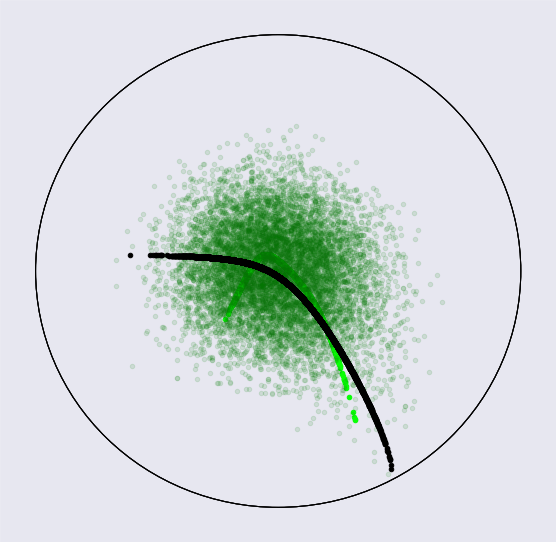}
    
    \caption{True submanifold (light green) and data points (green) generated for $n=10k$ and $\log \sigma^2 = -2$. Learned submanifold (black).  First column: $\mathbb{R}^2$; Second column: $S^2$; Third column: $H_2$ in its Poincar\'e disk representation.}
    \label{fig:flat}
\end{figure}

\section{Comparison of rVAE with submanifold learning methods}\label{sec:comp}

We perform experiments on simulated datasets to compare the following submanifold learning methods: PGA, VAE, rVAE, and VAE projected back on the pre-specified manifold. We generate datasets on the sphere using model~(\ref{model:gvae}) where the function $f_\theta$ is a fully connected neural network with two layers, and softplus nonlinearity. The latent space has dimension 1, and the inner layers have dimension 2. We consider different noise levels $\log \sigma^2 = \{-10, -2, -1, 0\}$ and sample sizes $n\in \{10k, 100k\}$.

We fit PGA using the tangent PCA approximation. The architecture of each variational autoencoder - VAE, rVAE and VAE projected - has the capability of recovering the true underlying submanifold correctly. Details on the architectures are provided in the supplementary materials. Figure~\ref{fig:goodness} shows the goodness of fit of each submanifold learning procedure, in terms of the extrinsic 2-Wasserstein distance in the ambient Euclidean space $\mathbb{R}^3$. The PGA is systematically off, as shown in the Figures from the supplementary materials, therefore we did not include it in this plot.


\begin{figure}
    \centering
    \includegraphics[width=8.2cm]{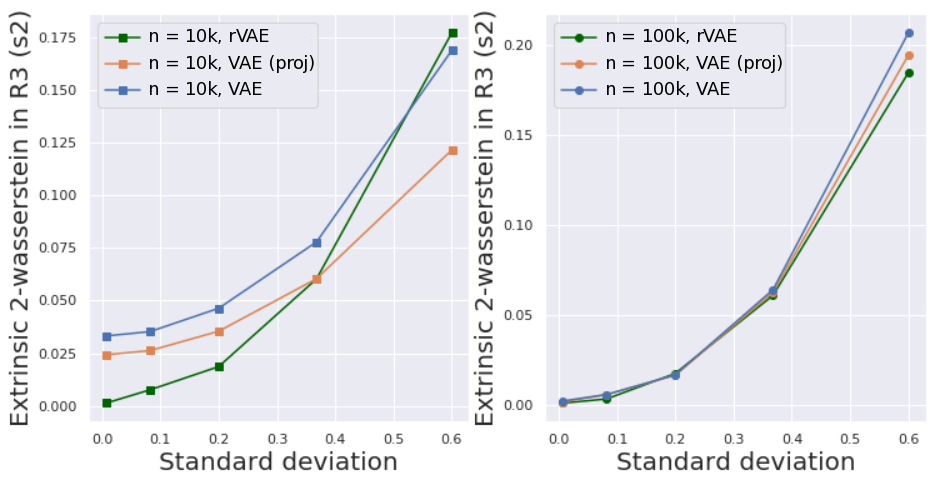}
    \caption{Quantitative comparison of the submanifold learning methods using the 2-Wasserstein distance in the embedding space $\mathbb{R}^3$. From left to right: ; quantitative comparison for $n=10k$ and different values of $\sigma$; quantitative comparison for $n=100k$ and different values of $\sigma$.}
    \label{fig:goodness}
\end{figure}


We observe that rVAE outperforms the other submanifold learning methods. Its flexibility enables to outperforms PGA, and its geometric prior allows to outperforms VAE. It also outperforms the projected VAE, although the difference in performances is less significative. Projected VAEs might be interesting for applications that do not require an intrinsic probabilistic model on the Riemannian manifold.

\section{Experiments on brain connectomes}\label{sec:spd}

In the last section, we turn to the question that has originated this study: do brain connectomes belong to a submanifold of the $\text{SPD}(N)$ manifold? We compare the methods of PCA, PGA, VAE and rVAE on resting-state functional brain connectome data from the ``1200 Subjects release" of the Human Connectome Project (HCP) \cite{VanEssen2013TheDavid}. We use $n=812$ subjects each represented by a $15\times15$ connectome. Details on the dataset are provided in the supplementary materials.


The VAE represents the brain connectomes as elements $x$ of the vector space of symmetric matrices and is trained with the Frobenius metric. In contrast, the Riemannian VAE represents the brain connectomes as elements $x$ of the manifold $\text{SPD}(N)$, which we equip with the Riemannian Log-Euclidean metric. We chose equivalent neural network architectures for both models. Details on the architectures and the training are provided in the supplementary materials. We perform a grid search for the dimension of the latent space over $L \in \{10, 20, 40, 60, 80, 100\}$. The latent dimension $L$ controls the dimension of the learned submanifold, as well as the model's flexibility. 



\begin{figure}[h]
\centering
\begin{center}
\includegraphics[page=4,height=4.3cm]{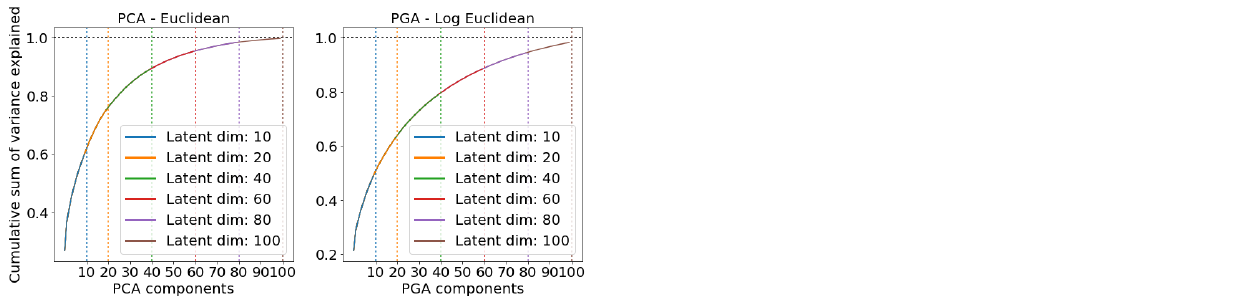}
\end{center}
\caption{Cumulative sum of variance captured by the principal components, for Principal Component Analysis (left) and Principal Geodesic Analysis (right).}
\label{fig:spd_comparison_linear}
\end{figure}

Results from PCA and PGA do not reveal any lower-dimensional subspace, see Figure~\ref{fig:spd_comparison_linear}. Figure~\ref{fig:spd_comparison} shows the results of VAE and rVAE. Both methods use only ~5 components from their latent space, even when $L$ is large. In the ambient space, they do not capture more than 34\% of the variance. Future work will investigate if this represents a feature of the connectomes space, truly equipped with a 5D nonlinear submanifold that represents \%30 of the variability, or if this is a failure mode of the rvAE and VAE.

\begin{figure}[h]
\centering
\begin{center}
\includegraphics[page=5,height=4.3cm]{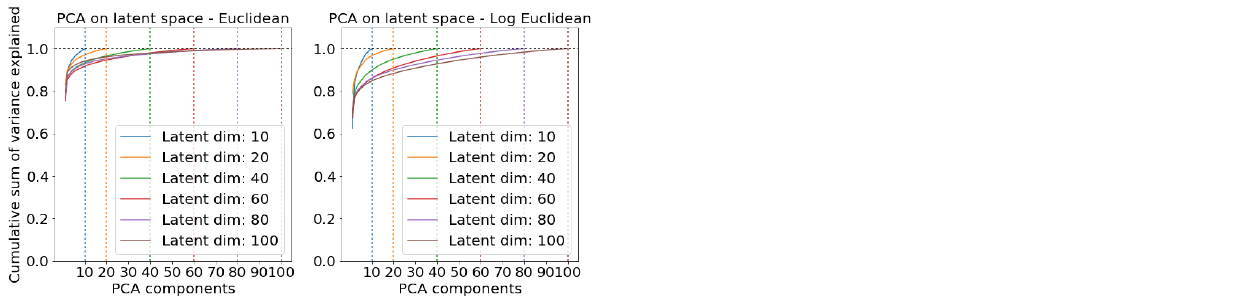}
\end{center}
\caption{Cumulative sum of variance captured by the principal components within the latent space, for the VAE (left); Right: Riemannian VAE (right).}
\label{fig:spd_comparison}
\end{figure}

\section{Conclusion}

We introduced the Riemannian variational autoencoder (rVAEs), which is an intrinsic generalization of VAE for data on Riemannian manifolds and an extension of probabilistic principal geodesic analysis (PPGA) to nongeodesic submanifolds. The rVAE variational inference method allows performing approximate, but faster, inference in PPGA. We provided theoretical and experimental results on rVAE using the formalism of weighted submanifold learning.

{\small
\bibliographystyle{ieee_fullname}
\bibliography{references}
}

\end{document}


\title{Supplementary Materials}

\author{First Author\\
Institution1\\
Institution1 address\\
{\tt\small firstauthor@i1.org}
\and
Second Author\\
Institution2\\
First line of institution2 address\\
{\tt\small secondauthor@i2.org}
}

\maketitle

\begin{abstract}
These are the supplementary materials to the paper: ``Learning Weighted Riemannian Submanifolds with Riemannian Variational Autoencoders".
\end{abstract}

\section{Elements of Riemannian Geometry}

\subsection{(Sub)manifolds, Weighted (Sub)manifolds}

We briefly review some elements of Riemannian geometry. They are useful to introduce the Riemannian VAE and to investigate the goodness of fit in terms of weighted submanifold learning. We refer to \cite{Postnikov2001} for details on the differential geometry. 

\begin{definition}[Riemannian manifold]
A Riemannian manifold $M$ is a differentiable manifold $M$ equipped with a metric $g$, which is a smoothly varying inner product on the tangent spaces of $M$. 
\end{definition}

In this paper, we assume all Riemannian manifolds considered are simply connected and complete \cite{Postnikov2001}. Within a given manifold, we will be interested in learning submanifolds. We recall the definition below.

\begin{definition}[Submanifold and embedded submanifold]
A subset $N$ of $M$ is a submanifold of $M$ if it is included in $M$ and is a manifold. The subset $N$ of the smooth manifold $M$ is called an embedded submanifold of $M$ if there exists a smooth manifold $X$ and a smooth embedding $f$ such that: $N = f(X)$. 
\end{definition}

A submanifold embedded in a Riemannian manifold inherits a Riemannian structure through the pull-back operation by $f$. On a Riemannian manifold $M$, the Riemannian metric allows to define a Riemannian volume element on $M$: $dM = \sqrt{\det(g(x))}dx$ in a local coordinate system $x$. We can define a Riemannian volume element on a Riemannian submanifold similarly. This allows to define weighted manifolds and submanifolds.

\begin{definition}[Weighted (sub)manifold]
Given a complete $d$-dimensional Riemannian manifold $(M, g)$ and a smooth probability distribution $\omega : M \rightarrow \mathbb{R}$, the weighted manifold $(M, \omega)$ associated to $M$ and $\omega$ is defined as the triplet:
\begin{equation}
    (M ,g,d\mu = \omega. dM),
\end{equation}
where $dM$ denotes the Riemannian volume element of $M$.
\end{definition}






\subsection{Geodesics, Exp and Log maps}

We introduce the main elements of differential geometry in order to perform computations in the following manifolds. Let $u, v$ be vector fields on $M$. The Riemannian metric $g$ on $M$ induces a notion of covariant derivative when considering the Levi-Civita connection associated to $g$ \cite{Postnikov2001}. Intuitively, the covariant derivative of $v$ in the direction of $u$, written $\nabla_u v$, represents the change of the vector field $v$ in the $u$ direction. Let $\gamma$ be a curve parameterized by $t$, as $\gamma : [0, 1] \mapsto M, t \mapsto \gamma(t)$. Its velocity is: $\dot \gamma = \frac{d\gamma}{dt}$. Let $u$ be a vector field $t \mapsto u(t)$ defined along $\gamma$. We can define the covariant derivative of $u$ along $\gamma$ to be $\frac{du}{dt} = \nabla_{\dot \gamma} u$.

\begin{definition}[Riemannian geodesic]
A curve $\gamma: [0, 1] \mapsto M$ is called a geodesic if it satisfies the equation $\nabla_{\dot \gamma} \dot \gamma = 0$, where $\nabla$ represents the covariant derivative of the Levi-Civita connection associated to the metric.
\end{definition}

Therefore, a geodesic is a curve that is parallel to itself. In this sense, a geodesic generalizes the notion of linear curve on Euclidean spaces. Geodesics on the sphere are the great circles. We can now generalize the notion of linear subspace to the notion of ``geodesic subspace", as a linear subspace of $\mathbb{R}^d$ is a submanifold that is ``geodesic" at every point.

\begin{definition}[Submanifold geodesic at a point]
A submanifold $N$ of M is said to be geodesic at $x \in M$ if all geodesics of $N$ passing through $x$ are also geodesics of $M$. 
\end{definition}

For any point $x \in M$ and tangent vector at this point $u \in T_xM$, there exits a unique geodesic $\gamma$, with initial conditions $\gamma(0) = x$ and $\dot \gamma(0) = u$. In general, the existence of the geodesic is only guaranteed locally around around $0$. As we assume the completeness of the manifold, the geodesics are defined on $\mathbb{R}$.

\begin{definition}[Riemannian exp map]
Let $\gamma$ be the unique geodesic $\gamma$, with initial conditions $\gamma(0) = x$ and $\dot \gamma(0) = u$. We define the Riemannian exponential map at $x$ as:
\begin{equation}
    \text{Exp}_x(u) = \gamma(1).
\end{equation}
\end{definition}
The exponential map takes a point and tangent vector as inputs and returns the point at time 1 obtained by ``shooting" along the geodesic with these initial conditions. The exponential map is a local diffeomorphism onto a neighbourhood of $x$ in $M$. 

\begin{definition}[Riemannian log map]
Let $V(x)$ be the largest neighbourhood on which the exponential map at $x$ is a diffeomorphism. Then within $V(x)$ the exponential map has an inverse, defined as the Riemannian log map, $\text{Log}_x : V(x) \mapsto T_xM$. 
\end{definition}

There is a class of manifolds, called Hadamard manifolds for which the Exp map is a global diffeomorphism. Mathematically speaking, Hadamard manifold are complete Riemannian manifold with non-positive sectional curvature \cite{Postnikov2001}.

\begin{theorem}[Cartan-Hadamard theorem]
For simply-connected Hadamard manifolds, the exponential map is globally diffeomorphic $T_xM$ onto $V(x) = M$ and the log map is defined on $M$, for any point $x \in M$.
\end{theorem}

This means that we can use the Exp map to transform the manifold into a vector space.



\subsection{Distance in $M$ and between submanifolds of $M$}

We define the notion of distance between points on the manifold $M$, and introduce the associated notion of distance between weighted submanifolds of $M$.

\begin{definition}[Riemannian distance]
For any point $y \in V(x)$, the Riemannian distance function is given by $d_M(x, y) = || \text{Log}_x(y)||_{g(x)}$ using the norm corresponding to the inner product $g(x)$ at $x$. 
\end{definition}

This notion allows to define a notion of distance between two embedded submanifolds of $M$, with Riemannian metric associated to their embedding $f$. We consider the 2-Wasserstein distance associated to the corresponding distributions, interpreted as singular distribution in $M$.

\begin{definition}[Wasserstein distance between submanifolds]
The 2-Wasserstein distance between probability measures $\mu$ and $\nu$ defined on $M$, is defined as:
\begin{equation}
    d_2(\mu, \nu) = 
    \left( 
    \inf_{\gamma \in \Gamma(\mu, \nu)} 
    \int_{M\times M} d_M(x_1, x_2)^2d\gamma(z_1, z_2)
    \right)^{1/2}
\end{equation}
where $\Gamma (\mu , \nu)$ denotes the collection of all measures on $M\times M$ with marginals $\mu$ and $\nu$  on the first and second factors respectively.
\end{definition}

This notion of distance will be our main evaluation metric when considering submanifold learning techniques.

\subsection{Examples of Riemannian manifolds}

 We recall that the $n$-dimensional hypersphere $S^{n}$ is defined by its embedding in the $(n+1)$-Euclidean space as:
\begin{equation}
S^{n} = 
\left\{
    x \in \mathbb{R}^{n+1}: x_1^2 + ... + x_{n+1}^2 = 1 
\right\}
\end{equation}
while the $n$-dimensional hyperbolic space $H_n$ is defined by its embedding in the $(n+1)$-dimensional Minkowski space, as:
\begin{equation}
    H_{n} = 
\left\{
    x \in \mathbb{R}^{n+1}: - x_1^2 + ... + x_{n+1}^2 = -1 
\right\}.
\end{equation}

\subsection{Riemannian geometry of SPD matrices}
The space of SPD matrices in $N$ dimensions is a manifold defined as: 
\begin{align*}
&\text{SPD}(N) \\
& \text{  }=\left\{S \in \mathbb{R}_{N \times N}: S^T = S, \forall z \in \mathbb{R}^N, z \neq 0, z^TSz > 0\right\}.
\end{align*}
It has dimension $D = \frac{N(N-1)}{2}$. The manifold $\text{SPD}(N)$ can be equipped with different Riemannian metrics \citep{Thanderwas2019}, for example the affine-invariant metric or the Log-euclidean metric. We note that $\text{SPD}(N)$ with either the Log-euclidean metric or the affine-invariant metrics is a Hadamard manifold, such that Corollary~\ref{cor} applies. 

Each SPD matrix is flatten and represented as a vector $x$ of dimension $D$ that corresponds to its upper half. We denote $\log(x)$ the vector representing the upper half of the matrix logarithm of the SPD matrix.

\section{Proof of Proposition 1}

We recall that $\mu = N(0, \mathbb{I}_L)$ and $\mu_T$ is the standard multivariate normal truncated at $10\sigma = 10$. We write $B_T$ the closed ball of $\mathbb{R}^L$ of radius 10. We write $C(\mathbb{R}^L, \mathbb{R}^D)$ the set of continuous functions from $\mathbb{R}^L$ to $\mathbb{R}^D$ or $C(B_T, \mathbb{R}^D)$ the set of continuous functions from $B_T$ to $\mathbb{R}^D$.

In what follows, for two functions $f_1, f_2$ in $C(B_T, \mathbb{R}^D)$, we consider the distance:
\begin{equation}
   d_{\infty}(f_1, f_2) = \sup_{z \in B_T} |f_1(z) - f_2(z)|
\end{equation}

\begin{proposition}
Let $(N, \nu)$ be an embedded weighted Riemannian submanifold of $M$, such that there exists an embedding $f$ that verifies: (i) $N = f(\mathbb{R}^L)$, i.e. $N$ is homeomorphic to $\mathbb{R}^L$, (ii) $\nu = f \ast \mu_T$ where $\mu_T$ is the truncated standard multivariate normal, (iii) $N \subset V(O)$ where $O=f(0)$ and $V(O)$ is the maximum domain of bijectivity of the Riemannian exponential of $M$ at $O$. Then, for any $\epsilon > 0$, there exists a Riemannian VAE with decoder $f_\theta$ parameterized by $\theta$ such that:
\begin{equation}
    d_{2}(N, N_\theta) < \epsilon
\end{equation}
where $d_2$ is the 2-Wasserstein distance for the weighted submanifolds.
\end{proposition}


\begin{proof}
Let $N$ be an embedded submanifold of $M$ verifying the assumptions. We write $f$ its embedding and $\tilde f = \text{Log}_\mu \circ f$.

Take $\epsilon < 0 $. Take $d \in [1, ..D]$ and consider $\tilde f_d$ the $d$-th component of $\tilde f$. We have: $\tilde f_d \in C(\mathbb{R}^L, \mathbb{R})$, which is continuous by the definition of an embedding. The weighted manifold $(N, \nu)$ can be equivalently represented by the embedding $\tilde f_{d, T}$ that is equal to $\tilde f$ on $B_T$ and $0$ everywhere else. We have $\tilde f_{d, T} \in C(B_T, \mathbb{R})$.
 
We use the framework and results from \cite{Hornik1991ApproximationNetworks} to show that there exists a neural network that approximates $\tilde f_d$ with precision $\epsilon$. In our notations, $L$ is the dimension of the input space (as opposed to $k$ in \cite{Hornik1991ApproximationNetworks}), $z$ is an element of the input space $\mathbb{R}^L$ (as opposed to $x$ in \cite{Hornik1991ApproximationNetworks}), and $K$ is the number of units in the hidden layer (as opposed to $n$ in \cite{Hornik1991ApproximationNetworks}). We define:
\begin{equation}
    M_L^{(K)}(\psi) = \left\{ h: \mathbb{R}^L \mapsto \mathbb{R} | h(z) = \sum_{k=1}^K \beta_k \psi (a'_k z - \theta_k) \right\},
\end{equation}
the space of functions represented by a fully connected neural network with one output unit and one hidden layer with $n$ hidden units, and:
\begin{equation}
    M_L (\psi) = \bigcup_{K=1}^{+\infty}  M_L^{(K)}(\psi)
\end{equation}
the space of functions represented by a fully connected neural network with one ouput unit and any number of hidden units. The $\psi$ represents the common activation function of all units. Using Theorem 2 from \cite{Hornik1991ApproximationNetworks}: if $\psi$ is continuous, bounded and nonconstant, then $M_L (\psi)$ is dense in $C(X, \mathbb{R})$ for any $X$ compact subset of $\mathbb{R}^L$. The component $\tilde f_d$ is in $C(B_T, \mathbb{R})$ and $B_T$ is a bounded closed subset of $\mathbb{R}^L$, thus a compact of $\mathbb{R}^L$. We consider a unbounded and nonconstant activation function $\psi$. Therefore, there exists a function $h_d \in M_L (\psi)$ such that:
\begin{equation}
    \sup_{z\in B_T} |\tilde f_d(z) - h_d(z)| < \epsilon.
\end{equation}
We fix the functions $h_1, ..., h_D$ that approximate the $D$ components of the function $\tilde f$. We define the function $h: \mathbb{R}^L \mapsto \mathbb{R}^D$ defined by $h = (h_1, ..., h_D)$. This function approximates $\tilde f$:
\begin{align*}\label{eq:approx}
&\sup_{z\in B_T} ||\tilde f(z) - h(z)||_2^2\\
    &\qquad= \sup_{z\in B_T}  \sum_{d=1}^D |\tilde f_d(z) - h_d(z)|^2 \\
    &\qquad\leq \sum_{d=1}^D  \sup_{z\in B_T} |\tilde f_d(z) - h_d(z)|^2\\
    &\qquad\leq \sum_{d=1}^D  \sup_{z\in B_T} |\tilde f_d(z) - h_d(z)|\sup_{z\in B_T} |\tilde f_d(z) - h_d(z)|\\
    &\qquad\leq D\epsilon^2
\end{align*}

The function $\tilde f$ can therefore be approximated by a neural network with precision $\epsilon$, and specifically by the neural network represented by $h$ and obtained by juxtaposing the neural networks defining the $h_d$, for $d=1, .., D$. We now write $f_\theta = h$, where $\theta$ represents the parameters of the neural network representing $h$. Additionally, we define: $f^M_\theta = \text{Exp}_O \circ f_\theta : \mathbb{R}^L \mapsto M$, where $O = f(0)$ and $\text{Exp}$ is the Riemannian exponential map in $M$. 

We show that the approximation result above translates into an approximation written in terms of 2-Wasserstein distances for the weighted submanifolds $(N = f(\mathbb{R}^L), \nu = f \ast N(0, I_L) = f\ast \mu)$ and $(N_\theta = f_\theta(\mathbb{R}^L, \nu_\theta = f^M_\theta \ast  N(0, I_L)= f^M_\theta \ast \mu)$. The squared 2-Wasserstein distance between the weighted submanifolds is by definition the 2-Wasserstein distance between the probability measures defined on them, $\nu$ and $\nu_\theta$, and can be written:
\begin{align*}
d_{W_2}(N, N_\theta)^2 
    &= \inf_{\gamma \in \Gamma(\nu, \nu_\theta)} 
    \int_{M\times M} d_M(x_1, x_2)^2d\gamma(x_1, x_2)
\end{align*}
where $\Gamma(\nu, \nu_\theta)$ is the collection of all measures on $M \times M$ with marginals $\nu$ and $\nu_\theta$ on the first and second factors respectively. As $\nu$ and $\nu_\theta$ have support on $N$ and $N_\theta$ respectively, this can equivalently be written:
\begin{align*}
d_{W_2}(N, N_\theta)^2 
    &=\inf_{\gamma \in \Gamma(\nu, \nu_\theta)} 
    \int_{N\times N_\theta} d_M(x_1, x_2)^2d\gamma(x_1, x_2).
\end{align*}
The functions $f$ is injective on its image by definition of an embedding. The function $f_\theta^M$ is assumed to be injective on its image (TODO: always true for NN)? Thus we can consider the homeomorphism:
\begin{equation}
    (f^{-1}, (f_\theta^{M})^{-1}): N \times N_\theta \mapsto \mathbb{R}^L. \times \mathbb{R}^L
\end{equation}
This homeomorphism is a global chart of the submanifold $N \times N_\theta$ of $M\times M$. We write the above integral in this chart. In other words, we use $(z_1, z_2) = (f^{-1}(x_1), (f_\theta^{M})^{-1}(x_2))$ to denote the coordinates of the point $(x_1, x_2)$ in the chart:
\begin{align*}
&d_{W_2}(N, N_\theta)^2 \\
    &\text{    }=\inf_{\gamma \in \Gamma(f^{-1}\ast \nu, (f_\theta^{M})^{-1}\ast \nu_\theta)} \\
   &\text{    }\qquad\qquad \int_{\mathbb{R}^L \times \mathbb{R}^L} d_M(f(z_1), f_\theta^M(z_2))^2d\gamma(z_1, z_2)\\
   &\text{    }=\inf_{\gamma \in \Gamma(\mu, \mu)} \\
   &\text{    }\qquad\qquad \int_{\mathbb{R}^L \times \mathbb{R}^L} d_M(f(z_1), f_\theta^M(z_2))d\gamma(z_1, z_2)
\end{align*}
where we can express the probability density $\gamma$ in this chart:
\begin{equation}
    d\gamma(z_1, z_2) = \gamma(z_1, z_2)dM(z_1) dM(z_2)
\end{equation}
where $dM(z_1) = \sqrt{\det(g(z_1)}dz_1$ and $dM(z_2) = \sqrt{\det(g(z_2)}dz_2$ refer to the volume element of $M \times M$ in this chart. (TODO: volume element of $N \times N_\theta$? what to do with the dM?). The notation $\Gamma(\mu, \mu)$ refers to collection of all measures on $\mathbb{R}^L \times \mathbb{R}^L$ with marginals $\mu$ and $\mu$ on the first and second factors respectively, with respect to the Euclidean measure of $\mathbb{R}^L$ (and not the Riemannian measure). We consider the following element $\tilde \gamma$ of $\Gamma(\mu, \mu)$: 
\begin{equation}
    \tilde \gamma(z_1, z_2) = \mu(z_1) \delta_{z_1=z_2} = \mu(z_2) \delta_{z_1=z_2}.
\end{equation}
By the property of the inf, we get:
\begin{align*}
d_{W_2}(N, N_\theta)^2
    &\leq  \int_{\mathbb{R}^L \times \mathbb{R}^L} d_M(f(z_1), f_\theta^M(z_2))^2\tilde d\tilde\gamma(z_1, z_2)\\
    &= \int_{\mathbb{R}^L} d_M(f(z), f^M_\theta(z))^2 d\mu(z)
\end{align*}


The function $\text{Exp}_O$ is continuous on the compact $K_1 = \tilde f(B_T)^1$ which corresponds to $\tilde f(B_T)$ but extended in $L_2$ norm of $\sqrt{D}$ in all directions. This guarantees that $f_\theta(z) \in K$ for $z\in B_T$ as soon as $f_\theta$ correspond to an $\epsilon < 1$. Therefore, the function $\text{Exp}_O$ is Lipshitz on $K$ and there exists a constant $C_{K_\epsilon}$ such that for any $y_1, y_2 \in K$, we have:
\begin{equation}
    d_M\left(\text{Exp}_O(y_1), \text{Exp}_O(y_1)\right) < C_K ||y_1 - y_2||_2
\end{equation}
We apply this to $y_1 =\tilde f(z) $ and $y_2 = f_\theta(z)$ for any $z\in B_T$:
\begin{align*}
    d_M\left(\text{Exp}_O(\tilde f(z)), \text{Exp}_O(f_\theta(z))\right) 
    &= d_M\left(f(z), f_\theta(z)^M\right)\\
    &< C_K ||\tilde f(z) -  f_\theta(z)||_2
\end{align*}
Applying the square function:
\begin{align*}
    d_M^2\left(\text{Exp}_O(\tilde f(z)), \text{Exp}_O(f_\theta(z))\right) 
    &= d_M\left(f(z), f_\theta(z)^M\right)\\
    &< C_K ||\tilde f(z) -  f_\theta(z)||_2^2
\end{align*}
Thus:
\begin{align*}
    \sup_{z\in B_T} d_M^2\left(f(z), f_\theta(z)^M\right) 
    &< C_K \sup_{z\in B_T}||\tilde f(z) -  f_\theta(z)||_2^2 \\
    &\leq C_{K} D\epsilon^2
\end{align*}

We have:
\begin{align*}
d_{W_2}(N, N_\theta)^2 
&\leq  \int_{\mathbb{R}^L} d_M(f(z), f^M_\theta(z))^2 d\mu(z)\\
&\leq \int_{\mathbb{R}^L} C_K D\epsilon^2 d\mu(z)\\
&\leq C_K D \epsilon^2
\end{align*}
 
\end{proof}

\section{SECOND PROOF of Proposition 1}

\begin{proposition}
Let $(N, \nu)$ be an compact weighted Riemannian submanifold of $M$, that is embedded in $L$, a submanifold of $M$ homeomorphic to $\mathbb{R}^L$, i.e. such that there exists en embedding $f$ st $L = f(\mathbb{R}^L$. We further assume a locality condition, which is that there exists a point $\mu$ so that $N \subset V(\mu)$, the bijectivity domain of the Riemannian exponential of $M$ at $\mu$. Then, for any $0 < \epsilon < 1$, there exists a rVAE with decoder $f_\theta$ such that:
\begin{equation}
    d(N_\theta, N) < \epsilon
\end{equation}
where $N_\theta$ is the weighted manifold described by: $N_\theta = f_\theta(\mathbb{R}^L)), f_\theta(N(0, 1))$, ie associated to the decoder and the prior $N(0, 1)$ standard for VAEs.
\end{proposition}

We need the locality condition because of the formulation of the VAE, that uses the Riemannian epxonential of $M$ to define the generative m odel. If $M$ is Hadamard,  we dont need this conditiion.

\begin{proof}

Let $N, \nu$ be a compact weighted Riemannian submanifold of $M$ verifying the assumptions. We write $L = f^M(\mathbb{R}^L)$ the embedded Riemannian submanifold of $M$ homeomorphic to $\mathbb{R}^L$ that contains $N$ (ie the support). 

$N$ is compact and $(f^M)^{-1}$ is continuous, therefore we can define $C_N$ the compact of $\mathbb{R}^L$:
\begin{equation}
    C_N = (f^M)^{-1}(N)
\end{equation}
where $(f^M)^{-1}$ is the embedding defining $L$. We consider the following function, with compact support:
\begin{equation}
    f_{C_N} = Log_\mu f^M|_N
\end{equation}
We see that $N$ is an embedded submanifold of $M$ that is homeomorphic to $C_N$: $N = f^M(C_N)$ by the same embedding function as the one defining $L$. As $N \in V(\mu)$, we can consider the function:
\begin{align*}
    f: C_N &\mapsto T_\mu M\\
        z & \mapsto f(z) = \text{Log}_\mu f^M(z)
\end{align*}
We have $f \in C(C_N, \mathbb{R}^D)$, where $C_N$ is a compact subset of $\mathbb{R}^L$. Therefore, there exists a decoder, such that:
\begin{equation}
    \sup_{z\in C_N} ||f(z) - f_\theta(z)||^2_2 < D\epsilon.
\end{equation}
Outside of $C_N$, the function $f_\theta$ is still defined, because it needs to be in the family $\mathfrak{M}_L$. We now consider $f_\theta$ as a function of $\mathbb{R}^L$. We fix the decoder $f_\theta$ and build the associated Riemannian decoder, which is the function:
\begin{align*}
    f^M_\theta : \mathbb{R}^L &\mapsto M\\
                 z & \maspto \text{Exp}(\mu, f_\theta(z))
\end{align*}
which is defined on $\mathbb{R}^L$. $f^M_\theta$ will almost contain $N$, as $f^M_\theta(C_N)$ is as close as wanted to $N$. However, $f^M_\theta$ brings with it the whole $N(0, 1)$, which deborde de $N$. Can we modify the $f_\theta$ en amount, cote $\mathbb{R}^L$, so that the transport optimal from $f^{-1}\ast \nu$ and $N(0, 1)$ is performed there?

We consider the weighted submanifold associated to this Riemannian decoder:
\begin{equation}
    N_\theta  = (f^M_\theta(C_N), f^M_\theta(C_N) \ast N(0, 1) ).
\end{equation}

We show that we can modify $f_\theta$, while keeping a decoder parameterization (?) so that both match.

The squared 2-Wasserstein distance between the weighted submanifolds is by definition the 2-Wasserstein distance between the probability measures defined on them, $\nu$ and $\nu_\theta$, and can be written:
\begin{align*}
d_{W_2}(N, N_\theta)^2 
    &= \inf_{\gamma \in \Gamma(\nu, \nu_\theta)} 
    \int_{M\times M} d_M(x_1, x_2)^2d\gamma(x_1, x_2)
\end{align*}
where $\Gamma(\nu, \nu_\theta)$ is the collection of all measures on $M \times M$ with marginals $\nu$ and $\nu_\theta$ on the first and second factors respectively. As $\nu$ and $\nu_\theta$ have support on $N$ and $N_\theta$ respectively, this can equivalently be written:
\begin{align*}
d_{W_2}(N, N_\theta)^2 
    &=\inf_{\gamma \in \Gamma(\nu, \nu_\theta)} 
    \int_{N\times N_\theta} d_M(x_1, x_2)^2d\gamma(x_1, x_2).
\end{align*}

The functions $f^M$ is injective on its image by definition of an embedding. The function $f_\theta^M$ is assumed to be injective on its image \textbf{(we assume this is true because it is close to $f^M$ which is injective and Exp is isomophirsms on $N \subset V(\mu)$)}? Thus we can consider the homeomorphism:
\begin{equation}
    (f^{-1}, (f_\theta^{M})^{-1}): N \times N_\theta \mapsto C_N \times C_N
\end{equation}
This homeomorphism is a ``global chart" (non sure if we can talk about chart, as they usually refer to homeopmorphisms to $\mathbb{R}^j$) of the submanifold $N \times N_\theta$ of $M\times M$. 

We write the above integral in this chart. In other words, we use $(z_1, z_2) = (f^{-1}(x_1), (f_\theta^{M})^{-1}(x_2))$ to denote the coordinates of the point $(x_1, x_2)$ in the chart:
\begin{align*}
&d_{W_2}(N, N_\theta)^2 \\
    &\text{    }=\inf_{\gamma \in \Gamma(f^{-1}\ast \nu, (f_\theta^{M})^{-1}\ast \nu_\theta)} \\
   &\text{    }\qquad\qquad \int_{C_N \times C_N} d_M(f(z_1), f_\theta^M(z_2))^2d\gamma(z_1, z_2)\\
   &\text{    }=\inf_{\gamma \in \Gamma(\mu, N(0, 1))} \\
   &\text{    }\qquad\qquad \int_{C_N \times C_N} d_M(f(z_1), f_\theta^M(z_2))d\gamma(z_1, z_2)
\end{align*}
where we can express the probability density $\gamma$ in this chart:
\begin{equation}
    d\gamma(z_1, z_2) = \gamma(z_1, z_2)dz_1dz_2
\end{equation}
and we have $(f_\theta^M)^{-1} \ast \nu_\theta = N(0, 1)$ by construction.

The notation $\Gamma(\mu,  N(0, 1))$ refers to collection of all measures on $C_N \times C_N$ with marginals $\mu = (f^M)^{-1} \ast \nu$ and  $N(0, 1)$ on the first and second factors respectively, with respect to the Euclidean measure of $\mathbb{R}^L$ (and not the Riemannian measure). 

We consider the following element $\tilde \gamma$ of $\Gamma(\mu, \mu)$: 
\begin{equation}
    \tilde \gamma(z_1, z_2) = \mu(z_1) \delta_{z_1=z_2} = \mu(z_2) \delta_{z_1=z_2}.
\end{equation}
By the property of the inf, we get:
\begin{align*}
d_{W_2}(N, N_\theta)^2
    &\leq  \int_{\mathbb{R}^L \times \mathbb{R}^L} d_M(f(z_1), f_\theta^M(z_2))^2\tilde d\tilde\gamma(z_1, z_2)\\
    &= \int_{\mathbb{R}^L} d_M(f(z), f^M_\theta(z))^2 d\mu(z)
\end{align*}


The function $\text{Exp}_O$ is continuous on the compact $K_1 = \tilde f(B_T)^1$ which corresponds to $\tilde f(B_T)$ but extended in $L_2$ norm of $\sqrt{D}$ in all directions. This guarantees that $f_\theta(z) \in K$ for $z\in B_T$ as soon as $f_\theta$ correspond to an $\epsilon < 1$. Therefore, the function $\text{Exp}_O$ is Lipshitz on $K$ and there exists a constant $C_{K_\epsilon}$ such that for any $y_1, y_2 \in K$, we have:
\begin{equation}
    d_M\left(\text{Exp}_O(y_1), \text{Exp}_O(y_1)\right) < C_K ||y_1 - y_2||_2
\end{equation}
We apply this to $y_1 =\tilde f(z) $ and $y_2 = f_\theta(z)$ for any $z\in B_T$:
\begin{align*}
    d_M\left(\text{Exp}_O(\tilde f(z)), \text{Exp}_O(f_\theta(z))\right) 
    &= d_M\left(f(z), f_\theta(z)^M\right)\\
    &< C_K ||\tilde f(z) -  f_\theta(z)||_2
\end{align*}
Applying the square function:
\begin{align*}
    d_M^2\left(\text{Exp}_O(\tilde f(z)), \text{Exp}_O(f_\theta(z))\right) 
    &= d_M\left(f(z), f_\theta(z)^M\right)\\
    &< C_K ||\tilde f(z) -  f_\theta(z)||_2^2
\end{align*}
Thus:
\begin{align*}
    \sup_{z\in B_T} d_M^2\left(f(z), f_\theta(z)^M\right) 
    &< C_K \sup_{z\in B_T}||\tilde f(z) -  f_\theta(z)||_2^2 \\
    &\leq C_{K} D\epsilon^2
\end{align*}

We have:
\begin{align*}
d_{W_2}(N, N_\theta)^2 
&\leq  \int_{\mathbb{R}^L} d_M(f(z), f^M_\theta(z))^2 d\mu(z)\\
&\leq \int_{\mathbb{R}^L} C_K D\epsilon^2 d\mu(z)\\
&\leq C_K D \epsilon^2
\end{align*}
 
\end{proof}

\section{Proof of Proposition 2}

\begin{proposition}
Suppose the encoder is flexible enough to encode $f_\theta^{-1}$. Then, the weighted submanifold learned by the VAE and the Riemannian VAE with no regularization are geometrically consistent.
\end{proposition}

\begin{proof}
The criterion that is optimized during the training of the VAE is as follows. The VAE seeks to minimize the expected squared distance from the $x_i$ to the $f_\hat{\theta}(\hat z)$, where $z$ is distributed as $q_{\phi}(z_i|x_i)$. This writes:
\begin{equation}
   \hat \theta_{VAE} = \text{argmin} \sum_{i=1}^n \mathbb{E}_{z \sim q_{\phi}(\bullet|x_i)}||x_i - f_{\hat{\theta}}(z)||^2
\end{equation}

This would be equivalent to minimizing the orthogonal square distance if we have: $\mathbb{E}_{z \sim q_{\phi}(\bullet|x_i)}||x_i - f_{\hat{\theta}}(z)||^2 = ||x_i - \pi^\perp_N(x_i)||^2$.

We consider the function of $\theta$:
\begin{align}
    \tilde p(\theta, \phi, x_i) &= \mathbb{E}_{z \sim q_{\phi}(\bullet|x_i)}||x_i - f_{{\theta}}(z)||^2
\end{align}
For $\sigma = 0$, the generative model becomes: $x = f_\theta^*(z)$ with $z \sim N(0, I)$. We have: $p(x|z) = \delta_{f_\theta^*(z)}$ and $p(z|x) = \delta_{f_\theta^{-1}^*(x)}$, assuming that $f$ is of full rank. The $\klpost = \int_z \varpost \log \frac{\varpost}{\post}dz$ forces the variational posterior to be zero where the posterior is $0$. As a consequence, $\varpost = \delta_{f_\theta^{-1}^*(x)}$ if the encoder is expressive enough to encode the function $f_\theta^{-1}^*$. Otherwise, it will be a delta at a value as close as possible to $f_\theta^{-1}^*(x)$, let's call $\mu_\phi(x)$ this value. Thus, for $\sigma = 0$:
\begin{align}
    \tilde p(\theta, \phi, x_i) 
    &= \mathbb{E}_{z \sim q_{\phi}(\bullet|x_i)}||x_i - f_{{\theta}}(z)||^2\\
    &= ||x_i - f_{{\theta}}(\mu_\phi(x_i))||^2\\
    &= ||x_i -f_{{\theta}} \circ f_{{\theta^*}}^{-1} (x_i)||^2 \qquad \text{if $\mu_\phi(x) = f_\theta^{-1}^*(x)$}\\
    &= 0
\end{align}
Therefore, the model reaches its minimizers:
\begin{equation}
    \left\{ \hat\theta | f_{\hat\theta} = f_{\theta}\right\}
\end{equation}

Now using the expression of the the 2-Wasserstein distance in the global chart of the submanifold $N_\theta \times N_{\hat\theta}$ of $M\times M$. In other words, we use $(z_1, z_2) = (f^{-1}(x_1), (f_\theta^{M})^{-1}(x_2))$ to denote the coordinates of the point $(x_1, x_2)$ in the chart:
\begin{align*}
&d_{W_2}(N, N_\theta)^2 \\
    &\text{    }=\inf_{\gamma \in \Gamma(f^{-1}\ast \nu, (f_\theta^{M})^{-1}\ast \nu_\theta)} \\
   &\text{    }\qquad\qquad \int_{\mathbb{R}^L \times \mathbb{R}^L} d_M(f(z_1), f_\theta^M(z_2))^2d\gamma(z_1, z_2)
\end{align*}
which is 0. TODO: Does not depend on $n$?!

\end{proof}

\section{Theoretical computations for 1D model}\label{app:1d}

We use * to denote the true parameters and we have $\truetheta = \truew$. The true distribution of the data is thus assumed to be: $p_{\truetheta}(x) =  N \left(0, \truew\truew^T + \mathbb{I}\right)$. The true posterior of the latent variables $z$ is tractable and writes: $p_{\truetheta}(z | x) = N\left(\Sigma \truew^T x, \Sigma = (I+\truew^T \truew)^{-1}\right)$. Thus, this model can be learned through the EM algorithm, which converges to a maximum likelihood estimator of $\truew$, under regularity assumptions, and computes the associated tractable posterior.

The case $D=1$ and $L=1$ allows to perform all computations in closed forms, as well as to represent the space of parameters $(\theta, \phi) = (w, \phi)$ in 2D as in Figure~\ref{fig:cv1d}. We set $\truew = 2$ and a sample size $n=1000$. Figure~\ref{fig:cv1d} shows the landscape the ELBO as a function of $(\theta, \phi)$. The vertical green line represents the true value of the parameter: $\truew = 2$, while the vertical blue lines represent the maximum likelihood estimates (MLE) of $\truew$. We note that there are two solutions, showing unidentifiability of the problem because of the symmetry in model~(\ref{eq:model1d}). The black curve represents the optimal variational distribution $q_\phi(z|x)$, for a given parameter $w$ of the generative model. The grey curve represents the optimal member of the true posterior $p_w(z|x)$, for a given variational posterior $q_\phi(z|x)$. We denote $\hat \sigma^2 = \frac{1}{n} \sum_{i=1}^n x_i^2$. The optimal solutions would be $\wopt = \wmle = \sqrt{\hat \sigma^2 -1}$ and $\phiopt = \phiopt(\wmle)$, i. e. the intersection of the blue line and the black curve, or symmetric solution $w=-\wmle$. The VAE algorithm maximizes the ELBO and converges to one of the two red dots. Specifically, it converges to $\welbo = \sqrt{\frac{\hat\sigma^2}{2} - 1}$ and associated $\phielbo = \phiopt(\welbo)$.

This section presents the derivation of the 1D theoretical study of the VAE, IWAE, and VIWAE algorithms, in Section 4 of the paper. In this simple example, we can perform the computations in closed forms.

\begin{lemma}
The log-likelihood of the generative model of Section 4 writes:
\begin{align*}
    L 
    &= -\frac{1}{2}\log(2\pi) - \frac{1}{2}\log(1+w^2) -\frac{\hat{\sigma^2}}{2(1+w^2)} \qquad \text{where: $\hat{\sigma^2} = \frac{1}{n}\sum_{i=1}^n x_i^2$.}
\end{align*}
A maximum likelihood estimator (MLE) for $w$ verifies: $w^2 = \hat{\sigma^2} - 1$. There are two MLEs: $\hat{w} = \sqrt{\hat{\sigma^2} - 1}$ and $\hat{w} = - \sqrt{\hat{\sigma^2} - 1}$. The model is unidentifiable.
\end{lemma}

\begin{proof}
The log-likelihood writes:
\begin{align*}
    L 
    &= \mathbb{E}_{\datadist}\left[\log p(x) \right]\\
    &= \mathbb{E}_{\datadist}\left[\log N(0, 1 + w^2)\right]\\
    &= \mathbb{E}_{\datadist}\left[-\frac{1}{2}\log(2\pi(1+w^2)) - \frac{x^2}{2(1+w^2)} \right]\\
    &= -\frac{1}{2}\log(2\pi) - \frac{1}{2}\log(1+w^2) -\frac{1}{2(1+w^2)}\mathbb{E}_{\datadist} x^2\\
    &\simeq -\frac{1}{2}\log(2\pi) - \frac{1}{2}\log(1+w^2) -\frac{1}{2(1+w^2)}\frac{1}{n}\sum_{i=1}^n x_i^2 \\
    &= -\frac{1}{2}\log(2\pi) - \frac{1}{2}\log(1+w^2) -\frac{\hat{\sigma^2}}{2(1+w^2)} \\
    &\qquad\qquad \text{where: $\hat{\sigma^2} = \frac{1}{n}\sum_{i=1}^n x_i^2$.}
\end{align*}
The MLE for $w$ is computed as follows:
\begin{align*}
    \frac{\partial L}{\partial w} = 0 
    &\Longleftrightarrow 
    \frac{\partial L}{\partial w}
    \left(- \frac{1}{2}\log(1+w^2) -\frac{\hat{\sigma^2}}{2(1+w^2)} \right) = 0 \\
    &\Longleftrightarrow 
    \left(- \frac{1}{2} \frac{2w}{1+w^2} + \frac{\hat{\sigma^2}2.2w}{4(1+w^2)^2} \right) = 0 \\
    &\Longleftrightarrow 
    \left(1 - \frac{\hat{\sigma^2}}{1+w^2} \right) = 0 \\
    &\Longleftrightarrow 
    1 + w^2 = \hat{\sigma^2} \\
    &\Longleftrightarrow 
    w^2 = \hat{\sigma^2} - 1
\end{align*}
Thus, there are two MLEs: $\hat{w} = \sqrt{\hat{\sigma^2} - 1}$ and $\hat{w} = - \sqrt{\hat{\sigma^2} - 1}$. 
\end{proof}

\begin{lemma}
The Kullback-Leibler divergence between the approximate posterior and the true posterior is:
\begin{align*}
   & \text{KL}\left(q_\phi(z|x) \parallel p_w(z | x)\right) \\
 &= \text{KL}\left(N(\phi x, 1) \parallel N\left(\frac{wx}{1+w^2}, \frac{1}{1+w^2}\right)\right)\\
 &= - \frac{1}{2}\log(1+w^2) + \frac{1}{2}(1+w^2) + \frac{1}{2} \frac{\left(w - \phi - \phi w^2\right)^2}{1+w^2}x^2 \\
 &\qquad\qquad- \frac{1}{2}
\end{align*}
Furthermore, for a fixed $w$, the $\phi$ minimizing the above \text{KL} divergence is given by: $\phi = \frac{w}{1+w^2}$. 
\end{lemma}

\begin{proof}
The \text{KL} divergence between two multivariate Gaussians in $L$ dimensions $N(\mu_1, \Sigma_1)$ and $N(\mu_2, \Sigma_2)$ is:
\begin{align*}
 &   \text{KL}(\mu_1, \Sigma_1 \parallel \mu_2, \Sigma_2) \\
& = \frac{1}{2}\left( \log\frac{|\Sigma_2|}{|\Sigma_1|} - L + tr(\Sigma_2^{-1}\Sigma_1) + (\mu_2-\mu_1)\Sigma_2^{-1}(\mu_2-\mu_1) \right)
\end{align*}

Thus, the \text{KL} between the approximate posterior and the true posterior is:
\begin{align*}
    & \text{KL}\left(q_\phi(z|x) \parallel p_w(z | x)\right) \\
    & = \text{KL}\left(N(\phi x, 1) \parallel N\left(\frac{wx}{1+w^2}, \frac{1}{1+w^2}\right)\right) \\
    & = \frac{1}{2}\log\frac{1/(1+w^2)}{1} + \frac{1}{2/(1+w^2)} + \frac{\left(wx/(1+w^2) - \phi x\right)^2}{2/(1+w^2)} - \frac{1}{2}\\
    & = \frac{1}{2}\log\frac{1}{1+w^2} + \frac{1}{2}(1+w^2) + \frac{\left(wx/(1+w^2) - \phi x\right)^2}{2/(1+w^2)} - \frac{1}{2}\\
    & = - \frac{1}{2}\log(1+w^2) + \frac{1}{2}(1+w^2) + \frac{1+w^2}{2} \left(\frac{wx}{1+w^2} - \phi x\right)^2 - \frac{1}{2}\\
    & = - \frac{1}{2}\log(1+w^2) + \frac{1}{2}(1+w^2) + \frac{1}{2(1+w^2)} \left(wx - \phi x (1+w^2)\right)^2 - \frac{1}{2}\\
    & = - \frac{1}{2}\log(1+w^2) + \frac{1}{2}(1+w^2) + \frac{1}{2} \frac{\left(w - \phi - \phi w^2\right)^2}{1+w^2}x^2 - \frac{1}{2}
\end{align*}
\end{proof}

\begin{lemma}
The ELBO writes:
\begin{align*}
    & \elbo\\
    &= \logl - \mathbb{E}_{\datadist}\left[ \text{KL}\left(q(z|x) \parallel p(z|x)\right) \right] \\
    &= \frac{1}{2}(1 - \log(2\pi)) - \frac{1}{2}\log\left(\frac{1+w^2}{w^2}\right) \\
    &\qquad\qquad-\frac{1}{2}w^2 - \frac{\hat{\sigma^2}}{2}\left(\frac{1}{1+w^2} +(1-\phi w )^2\right)
\end{align*}
where: $\hat{\sigma^2} = \frac{1}{n}\sum_{i=1}^n x_i^2$. 
\end{lemma}

\begin{proof}
The ELBO writes:
\begin{align*}
    &\elbo\\
    &= \logl - \mathbb{E}_{\datadist}\left[ \text{KL}\left(q(z|x) \parallel p(z|x)\right) \right] \\
    &= -\frac{1}{2}\log(2\pi) - \frac{1}{2}\log(1+w^2) -\frac{\hat{\sigma^2}}{2(1+w^2)} \\
    &\qquad\qquad- \mathbb{E}_{\datadist} \left[- \frac{1}{2} - \log w + \frac{1}{2}w^2 + \frac{1}{2}(1 - \phi w)^2 x^2\right] \\
    &= -\frac{1}{2}\log(2\pi) - \frac{1}{2}\log(1+w^2) -\frac{\hat{\sigma^2}}{2(1+w^2)} \\
    &\qquad\qquad + \frac{1}{2} + \log w - \frac{1}{2}w^2 - \frac{1}{2}(1 - \phi w)^2 \mathbb{E}_{\datadist} \left[x^2\right] \\
    &=\frac{1}{2}(1 - \log(2\pi)) + \log w - \frac{1}{2}\log(1+w^2) -\frac{1}{2}w^2 
    \\
    &\qquad\qquad\qquad\qquad\qquad- \frac{\hat{\sigma^2}}{2}\left(\frac{1}{1+w^2} +(1-\phi w )^2\right) \\
    &= \frac{1}{2}(1 - \log(2\pi)) - \frac{1}{2}\log\left(\frac{1+w^2}{w^2}\right) -\frac{1}{2}w^2 \\
    &\qquad\qquad\qquad\qquad\qquad- \frac{\hat{\sigma^2}}{2}\left(\frac{1}{1+w^2} +(1-\phi w )^2\right)
\end{align*}
\end{proof}

\begin{lemma}
Fixing $w$, the argument minimum of $\text{KL}\left(q_\phi(z|x) \parallel p_w(z | x)\right)$ for $\phi$ is $\phiopt(w) = \frac{w}{1+w^2}$. The function $w \rightarrow \phiopt(w)$ is represented by the black curve on Figure 1 of the paper. This is the optimal variational distribution, given a parameter $w$ of the generative model. Fixing $\phi$, the argument minimum of $\text{KL}\left(q_\phi(z|x) \parallel p_w(z | x)\right)$ for $w$ is $\wopt(\phi) = \frac{\phi\hat{\sigma^2}}{\phi^2\hat{\sigma^2} + 1}$. The function $\phi \rightarrow \wopt(\phi)$ is represented by the grey curve on Figure 1 of the paper. This is the optimal ``true" posterior, given a value of the variational posterior.
\end{lemma}

\section{Experimental results for 1D model}

This section presents additional figures regarding the experiments on the 1D model. Figure~\ref{fig:landscapes} presents the landscape of the log-likelihood, the negative KL divergence to the true posterior and the ELBO, as computed in the previous section. The ELBO is the sum of the previous two.

The VAE algorithm optimizes the ELBO and is, therefore, suboptimal with respect to parameter learning. We would like to: first, find the optimal $\theta$; then, compute the optimal $\phi$ corresponding to this $\theta$: the $\phi$ on the black curve on Figure~\ref{fig:landscapes}. However, we optimize the KL with respect to both parameters simultaneously and the two parameters collapse towards another. In Figure~\ref{fig:landscapes}, both are attracted to $0$. As the ELBO has the negative KL divergence as one of its terms, the parameter $\theta$ is attracted towards this optimal point of the negative KL. Another way of seeing it is as follows: the negative KL takes different values on the black curve, with a maximum close to 0. Thus, on this black curve, $(\theta, \phi)$ is attracted towards 0.

\begin{figure}[h!]
    \centering
    \includegraphics[height=3.7cm]{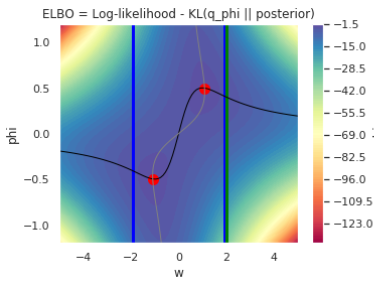} 
    \caption{Criteria as functions of the parameters $\theta$ and $\phi$. From left to right: log-likelihood, negative KL divergence to the true posterior and ELBO (nats).}
    \label{fig:landscapes}
\end{figure}


Figure~\ref{fig:distributions} shows the learned distributions of the data, as well as the learned variational distributions, for the 3 algorithms. As the VAE underestimates the parameter $w^*$, it also underestimates the variance of the data, as confirmed by the plot. Thus, even though its decoder was expressive enough to recover the distribution correctly, the VAE fails at doing so. The IWAE and the VIWAE do recover the distribution of the synthetic data. Similarly, the variational distribution learned by the VAE is not centered with respect to the true posterior distribution, whereas the IWAE show a better fit.

\section{Architecture details for the SPD dataset}

We present the details of the VAE and rVAE architectures for the SPD dataset examples on brain connectomes.

The decoder modeling $f_\theta$ is a fully connected neural network. We did not choose a convolutional network as the data are not images. The layers have dimensions $L$, $L^2$ and $D$, where $L$ is the dimension of the latent space and $D=120$ the dimension of the data space. In other words:
\begin{equation}
    f_\theta(x) = w^f_2g(w^f_1x + b^f_1) + b^f_2
\end{equation}
for $b^f_1 \in \mathbb{R}^{L^2}$, $w^f_1 \in \mathbb{R}^{L\times L^2}$, $b^f_2 \in\mathbb{R}^D$, $w^f_2 \in \mathbb{R}^{L^2\times D}$. We use ReLU as the activation function $g$. For the Riemannian VAE, we add the Riemannian exponential map $Exp_O$ at $O$ after the last layer, where $O$ is chosen to be the identity matrix $O = \text{Id}_{N\times N}$. As a consequence, $\text{Log}_O$ and $\text{Exp}_O$ correspond to the matrix logarithm and the matrix exponential. 

The encoder is a fully connected neural network, with three layers of dimensions $D=120$, $L^2$ and $L$. In other words:
\begin{align*}
    \mu_\phi(x) &= w^\mu_2g(w_1 x + b_1) + b^\mu_2\\
   \log \sigma^2_\phi(x) &= w^\sigma_2 g(w_1x + b_1) + b^\sigma_2
\end{align*}
for $b_1 \in \mathbb{R}^{L^2}$, $w_1 \in \mathbb{R}^{D\times L^2}$, $b^\mu_2, b^\sigma_2 \in\mathbb{R}^L$ and $w^\mu_2, w^\sigma_2 \in \mathbb{R}^{L^2\times L}$ . We use ReLU as the activation function $g$.

\begin{figure}
    \centering
    \includegraphics[width=5.2cm]{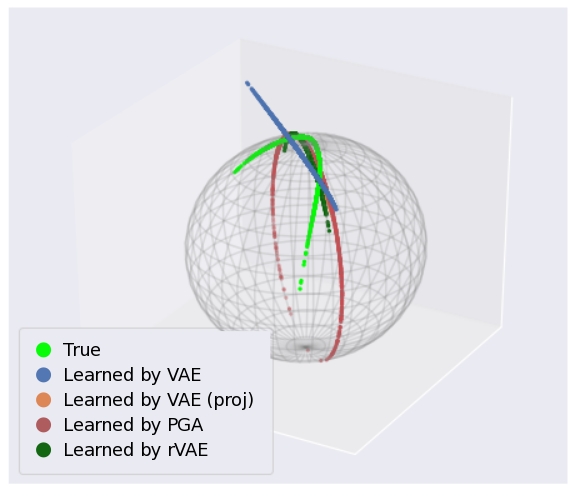}
    \caption{Visual comparison of the submanifold learning methods for $n=10k$ and $\sigma=0.35$}
    \label{fig:goodness}
\end{figure}

\begin{figure}
    \centering
    \includegraphics[height=2.6cm]{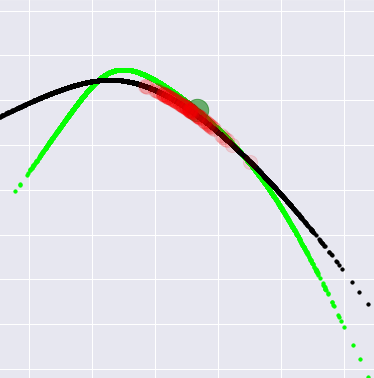}
    \includegraphics[height=2.6cm]{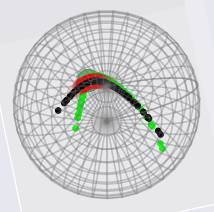}
    \includegraphics[height=2.6cm]{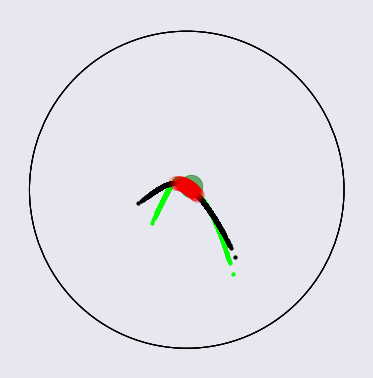}
    \caption{Posteriors and learned submanifolds (10k, -5)}
    \label{fig:posteriors}
\end{figure}

\subsection{VAE and rVAE Training}

The probability density functions of the noise models write in the Euclidean space:
\begin{equation}
    p_\theta(x|z) = \frac{1}{\sqrt{(2\pi)^D \sigma^{2D}}} \exp\left(- \frac{||x - f_\theta(z)||^2}{2\sigma^2}\right)
\end{equation}
and, in the Riemannian space:
\begin{equation}
    p_\theta(x|z) = \frac{1}{\sqrt{(2\pi)^D \sigma^{2D}}} \exp\left(- \frac{||\log(x) - f_\theta(z)||^2}{2\sigma^2}\right)
\end{equation}
We remark that, in this case, the integration constant only depends on $\sigma$, as the Riemannian manifold is Hadamard.

The losses write respectively:
\begin{align*}
    \mathcal{L}^{\text{VAE}}(x^{(i)}) 
    &=
    \frac{1}{2} \sum_{l=1}^L \left( 
        1 + \log(\sigma_l^{(i)})^2 - (\mu_l^{(i)})^2 - (\sigma_l^{(i)})^2
        \right)\\
       & \qquad - \log \sqrt{(2\pi)^D \sigma^{2D}} 
    - \frac{||x - f_\theta(z)||^2}{2\sigma^2}
\end{align*}
and:
\begin{align*}
    \mathcal{L}^{\text{rVAE}}(x^{(i)}) 
    &=
    \frac{1}{2} \sum_{l=1}^L \left( 
        1 + \log(\sigma_l^{(i)})^2 - (\mu_l^{(i)})^2 - (\sigma_l^{(i)})^2
        \right)\\
    & \qquad    - \log \sqrt{(2\pi)^D \sigma^{2D}} 
      - \frac{||\log(x) - f_\theta(z)||^2}{2\sigma^2}
\end{align*}

\subsection{PCA and PGA results}

Subjects effects. Ages? 22 to 36+.

\begin{figure}[h]
\centering
\begin{center}
\includegraphics[page=4,height=4.3cm]{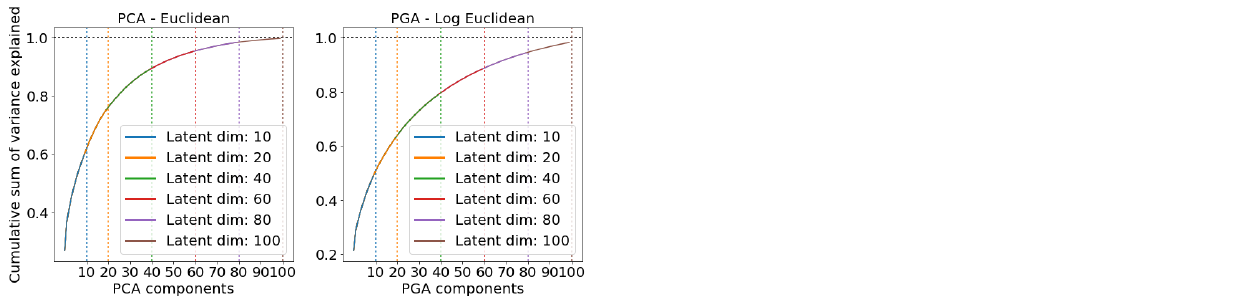}
\end{center}
\caption{Comparison of the submanifold learning methods on the resting states HCP connectomes. Top left: Principal Component Analysis; Top right: Principal Geodesic Analysis;.}
\label{fig:spd_comparison}
\end{figure}

Neither PCA nor PGA finds a lower-dimensional subspace of the data.

\begin{figure}[h]
\centering
\begin{center}
\includegraphics[width=8.4cm]{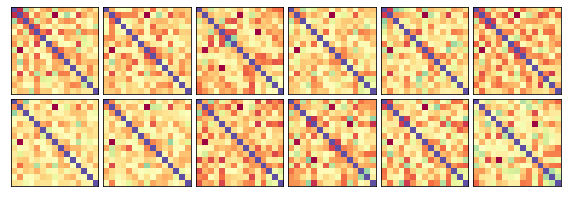}
\end{center}
\caption{Elements from the dataset of HCP brain connectomes \cite{VanEssen2013TheDavid}, with 15 nodes in each connectome.}
\label{fig:hcp}
\end{figure}

\subsection{Brain connectomes dataset}

We perform experiments on connectomes . The ``1200 Subjects release" includes 3T resting state functional MRI (rs-fMRI) imaging data from 1206 healthy young adult participants. We focus on the subdataset ``R812" which includes 812 subjects with rs-fMRI data reconstructed with the latest reconstruction algorithm.

We choose a parcellation of the brain into $N=15$ regions or ``nodes" and use the subject-specific sets of ``node time-series" provided by HCP. Each subject is represented by 15 time-series corresponding to the activation of each of the 15 brain regions through time. From each subject-specific time-series, we build the subject-specific parcellated connectome which is the 15x15 correlation matrix corresponding to the correlations between nodes. Each of the $n=812$ subject is therefore represented by a 15x15 symmetric positive definite (SPD) matrix. Elements randomly sampled from this dataset are shown in Figure~\ref{fig:hcp}.

{\small
\bibliographystyle{ieee_fullname}
\bibliography{references}
}